\documentclass[lettersize,journal]{IEEEtran}
\usepackage{amsmath,amsfonts,amssymb}
\usepackage{algorithmic}
\usepackage{algorithm}
\usepackage{array}
\usepackage{textcomp}
\usepackage{stfloats}
\usepackage{url}
\usepackage{verbatim}
\usepackage{graphicx}
\usepackage{cite}
\usepackage{bm}

\usepackage{subfigure}
\usepackage{booktabs}
\usepackage{bbding}
\usepackage{makecell}
\usepackage{CJKutf8}
\usepackage{etoolbox}
\makeatletter
\patchcmd{\@makecaption}
  {\scshape}
  {}
  {}
  {}
\makeatother
\usepackage{color}
\usepackage{hyperref}

\newtheorem{definition}{\textbf{Definition}}
\newtheorem{theorem}{\textbf{Theorem}}

\newtheorem{lemma}{\textbf{Lemma}}
\newtheorem{remark}{\textbf{Remark}}

\def\A{\mathcal{A}}

\begin{document}
\begin{CJK}{UTF8}{gbsn}
\title{Efficient Over-parameterized Matrix Sensing from Noisy Measurements via Alternating Preconditioned Gradient Descent}

\author{Zhiyu Liu, Zhi Han, Yandong Tang, Shaojie Tang, Yao Wang
         % stops a space
\thanks{Zhiyu Liu is with the State Key
Laboratory of Robotics, Shenyang Institute of Automation, Chinese Academy
of Sciences, Shenyang 110016, P.R. China, and also with the University of Chinese Academy
of Sciences, Beijing 100049, China (email: liuzhiyu@sia.cn).}% <-this % stops a space
\thanks{Zhi Han, Yandong Tang are with the State Key
Laboratory of Robotics, Shenyang Institute of Automation, Chinese Academy
of Sciences, Shenyang 110016, P.R. China (email: hanzhi@sia.cn; ytang@sia.cn).}
\thanks{Shaojie Tang is with Department of Management Science and Systems, State University of New York at Buffalo (e-mail: shaojiet@buffalo.edu).}
\thanks{Yao Wang is with the Center for Intelligent Decision-making and Machine
Learning, School of Management, Xi’an Jiaotong University, Xi’an 710049,
P.R. China. (email: yao.s.wang@gmail.com).}
}

%The paper headers


\maketitle

\begin{abstract}
We consider the noisy matrix sensing problem in the over-parameterization setting, where the estimated rank $r$ is larger than the true rank $r_\star$ of the target matrix $X_\star$. Specifically, our main objective is to recover a matrix $ X_\star \in \mathbb{R}^{n_1 \times n_2} $ with rank $ r_\star $ from noisy measurements using an over-parameterized factorization $ LR^\top $, where $ L \in \mathbb{R}^{n_1 \times r}, \, R \in \mathbb{R}^{n_2 \times r} $ and $ \min\{n_1, n_2\} \ge r > r_\star $, with $ r_\star $ being unknown. Recently, preconditioning methods have been proposed to accelerate the convergence of matrix sensing problem compared to vanilla gradient descent, incorporating preconditioning terms $ (L^\top L + \lambda I)^{-1} $ and $ (R^\top R + \lambda I)^{-1} $ into the original gradient. However, these methods require careful tuning of the damping parameter $\lambda$ and are sensitive to step size. To address these limitations, we propose the alternating preconditioned gradient descent (APGD) algorithm, which alternately updates the two factor matrices, eliminating the need for the damping parameter $\lambda$ and enabling faster convergence with larger step sizes. We theoretically prove that APGD convergences to a near-optimal error at a linear rate. We further show that APGD can be extended to deal with other low-rank matrix estimation tasks, also with a theoretical guarantee of linear convergence. To validate the effectiveness and scalability of the proposed APGD, we conduct simulated and real-world experiments on a wide range of low-rank estimation problems, including noisy matrix sensing, weighted PCA, 1-bit matrix completion, and matrix completion. The extensive results demonstrate that APGD consistently achieves the fastest convergence and the lowest computation time compared to the existing alternatives.
\end{abstract}

\section{Introduction}

Low-rank matrix sensing is a fundamental problem encountered in various fields, including image processing \cite{candes2011robust,li2019cloud}, phase retrieval \cite{vaswani2017low,nayer2021sample}, quantum tomography \cite{rambach2021robust}, among others. The primary objective is to recover a rank-$r_\star$ matrix $ X_\star\in\mathbb{R}^{n_1\times n_2}(r_\star \ll \min\{n_1,n_2\})$ from noisy linear measurements $\{(y_i,A_i)\}^m_{i=1}$ of the form 
\begin{equation}
    y_i=\langle A_i, X_\star\rangle + s_i , i=1,...,m, \label{equ:1.1}
\end{equation}
where $\{s_i\}^{m}_{i=1}$ denotes the unknown noise, which we assume to be sub-Gaussian with a variance proxy $\nu^2$. This model can be concisely expressed as 
$\bm{y}=\mathcal{A}(X_\star)+\bm{s}$, where $\mathcal{A}(\cdot): \mathbb{R}^{n_1 \times n_2}\mapsto \mathbb{R}^m$ denotes the measurement operator. A prevalent method for recovering the low-rank matrix $X_\star\in\mathbb{R}^{n_1\times n_2}$ involves solving the following problem:
\begin{equation}
\underset{X\in\mathbb{R}^{n_1 \times n_2}}{\min}   ||\A(X)-\bm{y}||_2^2,\ \operatorname{s.t.} \operatorname{rank}(X)\le r . 
\notag \end{equation} 
However, such an optimization problem is NP-hard due to the rank constraint. To address this challenge, researchers have proposed relaxing the rank constraint to a convex nuclear norm constraint \cite{recht2010guaranteed,candes2011tight,candes2012exact,candes2010power}. Although this kind of relaxation provides a tractable solution, it requires computing the matrix SVD, resulting in a significant increase in computational cost as the matrix size grows. To mitigate this computational overhead, a common approach is to decompose the matrix $X$ into a factorized form $LR^{\top}$, where $L\in\mathbb{R}^{n_1\times r},\ R\in\mathbb{R}^{n_2\times r}$, also known as the Burer-Monteiro method \cite{burer2003nonlinear,burer2005local}, and then solve the following problem:
\begin{equation}
    \underset{L \in\mathbb{R}^{n_1\times r},\ R \in\mathbb{R}^{n_2\times r}}{\arg \min } f(L,R) = \frac{1}{2} \|\A(LR^\top)-\bm{y}\|_2^2.
\label{equ:3}
\end{equation}

This problem can be efficiently solved by the vanilla gradient descent (GD) method \cite{tu2016low,zhuo2024computational,bhojanapalli2016global,jin2023understanding}:
\begin{equation}
L_{t+1}=L_t-\eta\nabla_Lf(L_t,R_t),\ R_{t+1}=R_t-\eta\nabla_Rf(L_{t},R_t).
\notag
\end{equation}

Despite significant progress in the field of non-convex matrix sensing, t challenges remain for vanilla gradient descent:
\begin{itemize}

\item \textbf{Over-parameterization}
The Burer-Monteiro method requires estimating the rank of the target matrix $X_\star$. However, a significant challenge is that, in practice, accurately estimating the rank of the matrix to be recovered is difficult. Therefore, it is typically assumed that the estimated rank is slightly larger than the true rank, that is, a situation known as over-parameterization. Previous work has shown that even under over-parameterization, accurate recovery of the matrix is still possible. However, over-parameterization can severely degrade the convergence rate of gradient descent algorithms, resulting in sub-linear convergence \cite{zhang2021preconditioned,zhang2023preconditioned,zhuo2024computational}.
    
\item \textbf{Poor conditioning}
It is well known that gradient methods are susceptible to the condition number $\kappa$ of the target matrix, defined as the ratio of the largest to the smallest singular value. Previous studies \cite{zheng2015convergent,zhang2023preconditioned} have shown that the number of iterations for gradient methods increases at least linearly with the condition number. Unfortunately, most practical datasets exhibit very large condition numbers. For instance, \cite{cloninger2014solving} noted that certain applications of matrix sensing can have condition number as high as \( \kappa = 10^{15} \), which can severely impact the practical application of GD.
\end{itemize}

\subsection{Preconditioning accelerates gradient descent}

In recent years, considerable attention has been given to addressing the aforementioned issues, with one key approach being the acceleration of vanilla GD under over-parameterization and ill-conditioning through preconditioning techniques. Essentially, preconditioning methods enhance the original gradient by adding right preconditioners, similar to the approach used in quasi-Newton methods. However, unlike Newton's method, preconditioning methods avoid computing the inverse of the large Hessian matrix (which has dimensions \((n_1 + n_2)r \times (n_1 + n_2)r\)). Instead, they only need to compute the inverses of two \( r \times r \) matrices, thereby significantly reducing the computational overhead.

Tong et al. \cite{tong2021accelerating} proposed ScaledGD for solving the matrix recovery problem in the exact-parameterized case, as shown in Equation (\ref{equ:4}):  
\begin{equation}
\begin{aligned}
\operatorname{\textbf{ScaledGD}}\ \ &L_{t+1} = L_t -\eta \nabla_L f(L_t,R_t)\cdot (R_t^\top R_t)^{-1}\\
    &R_{t+1} = R_t -\eta \nabla_R f(L_t,R_t)\cdot (L_t^\top L_t)^{-1}.
\end{aligned}
\label{equ:4}
\end{equation}

However, ScaledGD diverges in the over-parameterized situation. Therefore, to handle the over-parameterized case, several methods have been proposed \footnote{
These works focus on the case where $X$ is symmetric and positive semidefinite, with the corresponding loss function given by $f(L) = \| \mathcal{A}(LL^\top) - y \|_2^2$.
}, including ScaledGD$(\lambda)$ \cite{xu2023power}, PrecGD \cite{zhang2021preconditioned,cheng2024accelerating}, and NoisyPrecGD\footnote{A variant of PrecGD designed for noisy situations. For convenience, we refer to it as NoisyPrecGD.} \cite{zhang2024fast} as shown in the following Equation (\ref{equ:5}):
\begin{equation}
\begin{aligned}
    &   \operatorname{\textbf{ScaledGD($\lambda$)}}\ \	 L_{t+1} = L_t -\eta \nabla_L f(L_t)\cdot (L_t^\top L_t + {\lambda} I)^{-1} \\
	&\operatorname{\textbf{PrecGD}}\ \ \ \ \ \ \ \  L_{t+1} = L_t -\eta \nabla_L f(L_t)\cdot (L_t^\top L_t + {\lambda_t} I)^{-1}. 
\label{equ:5}
\end{aligned}
\end{equation}
A common feature of these methods is the inclusion of an additional damping term $\lambda I$, and the use of symmetric positive semi-definite matrix \( X = LL^\top \). The difference lies in the selection of the damping parameter \( \lambda \). ScaledGD$(\lambda)$ requires \( \lambda \) to be a fixed, very small constant, while PrecGD requires \( \lambda \) to change dynamically, i.e., $\lambda_t = \Theta(\|L_tL_t^\top - X_\star\|_F)$. NoisyPrecGD \cite{zhang2024fast} points out that both of these methods fail in the presence of noise. To address this, they propose an exponential decay adjustment: \( \lambda_{\text{new}} = \beta \lambda \), where $0<\beta<1$.

However, these methods all require careful tuning of an appropriate \( \lambda \) to achieve optimal results. Additionally, they only consider symmetric positive semi-definite matrices, which is a simpler case. These limitations significantly hinder the practical applicability of the existing methods. This raises the following question: \textbf{Can we develop an algorithm that does not rely on the damping term, removes the symmetric positive semi-definite constraint, and still converges to near-optimal error at a linear rate?}

\begin{table}[ht]
\caption{Comparison of related works in over-parameterized noisy matrix sensing. In the second column, the upper bounds or exact setting of step size in the previous work are listed. The fourth column indicates whether the asymmetric factorization is considered. The fifth column refers to whether the preconditioning method requires the damping parameter \(\lambda\). According to \cite{zhang2024fast}, $\frac{1}{L_1}=\min\left\{ \frac{L_{\delta}}{60\sqrt{2}(1+\delta)+25(1+\delta)^2},\frac{1}{7L_{\delta}} 
 \right\}$, where \( \delta \) is the rank-$2r$ RIP constant and \( L_\delta \) is some constants. It is clear that $\frac{1}{L_1}$ is a relatively small number.}
\centering
\setlength{\tabcolsep}{4pt} 
\begin{tabular}{cccccc}
\hline
methods & step size  & convergence rate & asymmetry & damping term \\ \hline

\cite{ding2022validation} & $\le \frac{1}{c\kappa^2 \sigma_1(X_\star)}$   & sub-linear   & \XSolidBrush  & \textbackslash \\ \hline
\cite{zhuo2024computational} & $= \frac{1}{100\sigma_1(X_\star)}$   & sub-linear & \XSolidBrush & \textbackslash \\ \hline
\cite{zhang2024fast} &  $\le \frac{1}{L_1}$   & linear & \XSolidBrush & \Checkmark\\ \hline
% \cite{cheng2024accelerating} & $\frac{1}{L_2}$  & local & linear & \Checkmark  & \Checkmark  \\ \hline 

ours &  $\le\frac{1}{1+\delta}$   & linear & \Checkmark  & \XSolidBrush \\ \hline

\end{tabular}
\label{table:1}
\end{table}

\subsection{Alternating helps: damping-free preconditioner}

To address the aforementioned question, we propose an alternating preconditioned gradient descent (APGD) algorithm to solve the over-parameterized matrix sensing problem. Many previous works have primarily considered the symmetric positive semi-definite case, which is often regarded as a simpler setting. Additionally, the favorable properties of symmetric positive semi-definite matrices can be leveraged to simplify the analysis.

However, we argue that asymmetric decomposition, compared to symmetric decomposition, offers the advantage of enabling more efficient and practical algorithm. This benefit arises from the alternating update. Specifically, after performing asymmetric decomposition on a given matrix, a natural approach is to alternately update the two matrices \cite{jain2013low,tanner2016low,lee2023randomly, gu2024low,ward2023convergence}. Notably, \cite{jia2024preconditioning} proved that alternating ScaledGD does not depend on a small step size, which has been a major inspiration for our work.
Inspired by \cite{jia2024preconditioning,tanner2016low}, we propose an APGD algorithm that combines alternating updates with preconditioning to solve the noisy asymmetric matrix sensing problem. We show that, after applying alternating update, the damping parameter in the preconditioner becomes unnecessary. It is worth noting that the APGD algorithm can also be applied to other low-rank matrix estimation problems, such as weighted PCA, matrix completion, and 1-bit matrix completion. As shown in the informal theorem, we not only provide convergence rate and error bounds for the noisy matrix sensing task, but also establish convergence guarantees for the APGD algorithm in the general case.
\begin{theorem}(Informal)
For the noisy over-parameterized matrix sensing problem, under some mild assumptions, starting from a initial point closed to the ground truth, APGD converges to the near-minimax error in a linear rate with high probability, i.e.,
$$
\|L_tR_t^\top-X_\star\|_F^2 \lesssim \max \left\{ Q_f ^{2t} \|L_0 R_0^\top - X_\star\|_F^2, \mathcal{E}_{opt} \right\},
$$
where $0<Q_f<1$, $\mathcal{E}_{opt}=C_e \frac{\nu^2rn\log n }{m}$, and $n = \max \{n_1,n_2\}$.

Moreover, for general low-rank matrix estimation problems, if the loss function \(g\) satisfies some mild conditions, APGD can achieve linear convergence when initialized sufficiently close to the ground truth, i.e.,  
$$
g(X_{t+1}) - g(X_\star) \le  Q_g  \, [g(X_t) - g(X_\star)],
$$
for some constant $0<Q_g<1$ which depends on the geometric properties of $g$ and step size.

\label{theorem:informal}
\end{theorem}

\begin{algorithm}[h]
\caption{Alternating Preconditioned Gradient Descent (APGD) for noisy matrix sensing}
\label{algorithm}
\textbf{Input:} Observation $\{y_i,\A_i\}_{i=1}^m$, step size $\eta$, estimated rank $r$.\\
{\textbf{Initialization}: Let $(\A^*\A(y))_r$ be the rank-$r$ approximation of $\A^*\A(y)$ and  $U_0S_0V_0^\top$ be the svd of $(\A^*\A(y))_r$. Then we set $L_0 = U_0S_0^{\frac{1}{2}},\ R_0 = V_0S_0^{\frac{1}{2}}$.}
\begin{algorithmic}[1] %[1] enables line numbers
\STATE \textbf{for} $t=0$ to $T-1$ \textbf{do}
\STATE \ \ \ \ \ $L_{t+1}=L_t-\eta\nabla_Lf(L_t,R_t)\cdot (R_t^\top R_t)^{\dagger}$
\STATE \ \ \ \ \ $R_{t+1}=R_t-\eta\nabla_Rf(L_{t+1},R_t)\cdot (L_{t+1}^\top L_{t+1})^\dagger$\\
 \ \ \ \ \  ($\dagger$ denotes the Moore-Penrose-Pseudo inverse)
\STATE \textbf{end for}
\STATE \textbf{return:} $X_T=L_TR_T^\top$
\end{algorithmic}
\end{algorithm}

We shall summarize the contributions of this paper as follows: 

\begin{itemize}
    \item We propose an alternating preconditioning algorithm for the asymmetric matrix sensing problem with noisy measurements. Compared to other precondition methods, APGD does not require a damping term in the preconditioner, thus eliminating the need for parameter tuning. Moreover, APGD is less sensitive to the step size and can converge faster with larger step sizes. All these make APGD more practical and efficient than the previous methods, and it can be extended to other low-rank matrix estimation problems.
    
    \item 
    We analyze the convergence properties of APGD  and prove that it converges to the near-optimal error at a linear rate. Our analysis highlights that the advantage of APGD over other methods lies in the alternating update, which decomposes the optimization into two sub-problems. This reduces the Lipschitz constant for each subproblem, therefore allowing for larger step size. It is worth noting that our analysis framework can be extended to other low-rank matrix estimation tasks. We show that APGD also achieves linear convergence for a variety of such problems.

    \item 
    We conduct a series of experiments demonstrating that APGD converges to near-optimal recovery error at the fastest rate compared with other works, and further possesses of better robustness against the choice of step size. In addition, simulation and real-data experiments on weighted PCA, 1-bit matrix completion, and matrix completion demonstrate the broad practical potential of APGD.

\end{itemize}

\section{Related work}

Recent research in matrix sensing has focused on fast non-convex algorithms, notably the Burer-Monteiro  factorization \cite{tu2016low,zhuo2024computational,chen2015fast,sun2016guaranteed}. Despite progress, gradient descent (GD) struggles with ill-conditioning and over-parameterization, prompting extensive studies. We present a comparison of several works most relevant to our approach in Table 1.

\textbf{Preconditioning}
Gradient-based methods are highly sensitive to the condition number of the matrix, with the iteration complexity of gradient descent (GD) scaling linearly with it—i.e., $\mathcal{O}(\kappa \log(1/\epsilon))$. As the condition number increases, the convergence rate of GD deteriorates significantly \cite{zheng2015convergent,zhang2023preconditioned}.  To address this issue, a growing body of research has focused on preconditioning techniques \cite{mishra2012riemannian,wei2016guarantees,mishra2016riemannian,tanner2016low,tong2021accelerating,zhang2021preconditioned,zhang2023preconditioned,zhang2022accelerating,bian2023preconditioned,jia2024globally,jia2024preconditioning,cai2021learned,cai2024deeply}. Tong et al. \cite{tong2021accelerating} proposed the ScaledGD algorithm for a range of low-rank matrix estimation problems and provided a detailed convergence analysis. However, ScaledGD is not applicable in over-parameterized regimes. To overcome this limitation, Zhang et al. \cite{zhang2021preconditioned,zhang2023preconditioned} introduced PrecGD for over-parameterized matrix sensing, and subsequently developed an improved version to handle noisy measurements \cite{zhang2024fast}. They also extended the preconditioning framework to the online matrix completion setting \cite{zhang2022accelerating}. Preconditioning methods have also been explored in robust matrix recovery. Tong et al. \cite{tong2021low} proposed the ScaledSM algorithm for recovery under $\ell_1$ loss, establishing local linear convergence guarantees. Building on this, Giampouras et al. \cite{giampouras2024guarantees} introduced the OPSA algorithm to accelerate robust recovery in over-parameterized scenarios.  
While most of these methods focus on local convergence, Xu et al. \cite{xu2023power} went further by establishing the global convergence of the ScaledGD($\lambda$) algorithm for over-parameterized matrix sensing. More recently, Jia et al. \cite{jia2024preconditioning} provided global convergence guarantees for both ScaledGD and AltScaledGD in the matrix factorization setting.

% Consequently, many studies have focused on addressing this issue \cite{mishra2012riemannian,wei2016guarantees,mishra2016riemannian}. \cite{tanner2016low} proposed an efficient alternating steepest descent (ASD) method and its scaled variant, ScaledASD, for the fixed rank matrix completion problem. \cite{tong2021accelerating} proposed a scaled gradient descent (ScaledGD) algorithm, applying it to various low-rank matrix estimation tasks such as matrix sensing, robust principal component analysis (RPCA), and matrix completion. Furthermore, \cite{tong2021low} introduced a scaled sub-gradient algorithm for solving non-convex and non-smooth low-rank matrix optimization problems. Moreover, \cite{zhang2022accelerating} further considered the problem of completing streaming data with extremely large condition numbers and proposed a preconditioned version of the stochastic gradient descent (SGD) algorithm, which significantly accelerates the convergence rate of the standard SGD algorithm.

\textbf{Over-parameterization} Earlier works \cite{tu2016low,tong2021accelerating,chen2015fast,li2018algorithmic} demonstrated that, under the exact rank assumption, gradient descent method could converge to the ground truth at a linear rate. However, since it is difficult to determine the exact rank of the matrix to be recovered in practice, recent research has focused on matrix recovery in the overestimated rank setting \cite{zhuo2024computational,li2018algorithmic,stoger2021small,soltanolkotabi2023implicit}.
Recent studies have shown that in over-parameterized settings, gradient descent \cite{zhuo2024computational} or subgradient descent \cite{ding2021rank} with spectral initialization can achieve sublinear convergence to the optimal solution. Furthermore, \cite{stoger2021small,xiong2024how,soltanolkotabi2023implicit,jin2023understanding,ma2022global} proved that using small initialization in such settings leads to linear convergence. However, small initialization typically requires a long time to escape saddle points. Overall, over-parameterization tends to slow down the convergence rate of gradient-based or subgradient-based algorithms. Studies by \cite{zhang2021preconditioned,zhang2023preconditioned,xu2023power,cheng2024accelerating,giampouras2024guarantees,zhang2024projected} have explored the issue of slow convergence in over-parameterized settings.

\textbf{Noisy matrix sensing}
For the noisy matrix sensing problem, some existing studies \cite{ma2022sharp,ma2023noisy,zhang2018primal,ma2023geometric} focus on landscape analysis, aiming to provide global guarantees on the maximum distance between any local minimum and the ground truth. Other works\cite{zhuo2024computational,ding2022validation,zhang2024fast}, including this paper, focus on analyzing the convergence rate and statistical error of algorithms. previous works have shown that vanilla gradient descent can achieve a statistical error of $\mathcal{O}(v^2 n r \log n)$, where \(r\) is the estimated rank. If we further assume that $r = \mathcal{O}(r_\star)$, then the resulting error differs from the minimax optimal error established in \cite{candes2011tight} by only a logarithmic factor. Ding et al. \cite{ding2022validation} showed that gradient descent with extremely small initialization and early stopping can achieve the optimal error. However, such approaches converge very slowly when the condition number is large, making them impractical in real-world scenarios. Zhang et al.\cite{zhang2024fast} proposed a preconditioned gradient descent algorithm for the noisy setting and proved that it achieves linear convergence up to a near-optimal error. However, their method requires tuning the damping parameter in the preconditioner and is limited to symmetric positive semidefinite matrices.

\section{Main results}
\subsection{Preliminaries}
\textbf{Notations} Singular values of a rank-$r$ matrix $X$ are donated as $\|X\|=\sigma_1(X)\ge\sigma_2(X)\ge\cdots \ge \sigma_r(X)>0$. We denote the condition number of the truth matrix $X_\star$ as $\kappa=\sigma_1(X_\star)/\sigma_{r_\star}(X_\star)$.
\begin{definition}(Restricted Isometry Property)
The linear map $\mathcal{A}(\cdot)$ is said to satisfies Restricted Isometry Property (RIP) with parameters $(r,\delta_r)$ if there exits constants $0\le \delta_r <1$ and $m>0$ such that for every rank-$r$ matrix $M$, it holds that
\begin{equation}
    (1-\delta_r)\|M\|_F^2 \le\|\A(M)\|_2^2 \le (1+\delta_r)\|M\|_F^2.
    \notag
\end{equation}
\end{definition}

RIP is a widely used condition in the field of compressed sensing, which states that the operator \(\mathcal{A}(\cdot)\) approximately preserves distances between low-rank matrices. In the absence of noise, we can establish a direct relationship between the loss function $||\mathcal{A}(LR^\top-X_\star)||_2^2$ and the recovery error $||LR^\top-X_\star||_F^2$. 

It is well known that if each measurement matrix $A_i$ consists of independent (sub-)Gaussian entries with zero mean and variance $1/m$, then the operator $\mathcal{A}$ satisfies the rank-$r$ Restricted Isometry Property (RIP) with constant $\delta > 0$, provided that the number of measurements satisfies $m \gtrsim r(n_1 + n_2)/\delta^2$; see \cite{candes2011tight} for details.

However, in the presence of noisy observations, the interference from noise prevents us from directly applying the RIP condition. Therefore, similar to \cite{zhang2024fast}, we utilize the following decomposition:
\begin{equation}
\begin{aligned}
&f(L_t,R_t) = \frac{1}{2} \| \mathcal{A}(L_t,R_t)-y \|_2^2 \\
&= \underbrace{\frac{1}{2}\| \mathcal{A}(L_t R_t^\top-X_\star) \|_2^2}_{f_c(L_t,R_t) }  + \frac{1}{2} \| s \|_2^2 - \frac{1}{2} \langle \A(L_tR_t^\top-X_\star), s \rangle. 
\notag
\end{aligned}
\end{equation}

Then, we can apply the RIP condition to derive the following inequality:
$$
(1-\delta_{2r+1})\|E_t\|_F^2 \le  f_c(L_t,R_t) \le (1+\delta_{2r+1})\|E_t\|_F^2,
$$
where $E_t=L_tR_t^\top - X_\star$.

\subsection{Main theorem for noisy matrix sensing}
Based on these preliminaries, we directly present the main result, with its detailed proof provided in the Appendix \ref{proof of the main results}.

\begin{theorem}
Suppose the following conditions hold: (1) each entry of the sensing matrix $A_i$ is independently drawn from the Gaussian distribution $\mathcal{N}(0, 1/m)$.
(2) the measurement number $m\ge C_\delta \frac{v^2 rn \log n }{ \sigma_{r_\star}(X_\star) \rho^2 \delta_{2r+1}^2}$ with constant $\delta_{2r+1}\le \frac{\rho}{8\kappa\sqrt{r_\star + r}},\ \rho\le \frac{1}{2}$; (3) the step size $\eta \le \frac{1}{(1+\delta_{2r+1})}$. Then solving the over-parameterized and noisy matrix sensing problem with algorithm 1, we have
$$
\|L_tR_t^\top-X_\star\|_F^2 \le \max \left\{ C_\delta Q_f^{2t} \|L_0 R_0^\top - X_\star\|_F^2, C_3\mathcal{E}_{opt} \right\},
$$ holds with probability at least $1-3 n^{-c_1} -2e^{-c_2 m\delta_{2r+1}} $，
where $Q_f = 1-\eta_c $, $$\eta_c = \tau \left ( \eta - \frac{\eta}{3}(1 + 2\eta(1+\delta_{2r+1}) ) \right),$$ $$\tau = \left( \sqrt{\frac{1-3\rho^2}{1-\rho^2}} -\sqrt{r+r_\star} \delta_{2r+1} \right)^2,$$ $C_\delta =\frac{1+\delta_{2r+1}}{1-\delta_{2r+1}}$, $\mathcal{E}_{opt}= C_e \frac{\nu^2rn\log n }{m}$, $n=\max\{n_1,n_2\}$, and $C_3=\frac{1}{\tau}+7$.
\label{main theorem}
\end{theorem}

\textbf{Recovery error}
Our recovery error $\mathcal{O}(\frac{\nu^2rn\log n }{m})$ is near-optimal up to a log factor, which is consistent with most existing works \cite{tu2016low,zhuo2024computational,zhang2024fast}. However, \cite{ding2022validation} proved that using small initializations, gradient descent can converge to the error of $\mathcal{O}(\nu^2 \kappa^2 \frac{r_\star n}{m})$. This error is independent of the over-rank $r$ and is optimal when the condition number is 1. However, in practical scenarios, the condition number is rarely equal to one. When it becomes large, the estimation error can increase significantly. In contrast, our error is independent of the condition number.

\textbf{Initialization} 
In our theoretical analysis, we require the initial point to be sufficiently close to the ground truth, a standard assumption commonly adopted in prior works \cite{zhang2021preconditioned,zhang2023preconditioned,zhang2024fast,zhuo2024computational,tong2021accelerating}. This condition can be easily satisfied via spectral initialization. It is important to note, however, that this requirement is primarily for theoretical guarantees. In fact, APGD is not sensitive to initialization and can converge reliably even from random starting points, as confirmed by the experimental results presented in section \ref{extension experiments}. Providing a theoretical guarantee of its global convergence is left as future work.

\textbf{Step size} APGD is highly robust to the step size; it only requires the step size to satisfy \( \eta< \frac{1}{1+\delta_{2r+1}} \). In contrast, other methods require the step size to be very small. In \cite{zhuo2024computational}, the step size is set to be $\eta=\frac{1}{100\sigma_1(X_\star)}$, which is a very small value. In \cite{ding2022validation}, the step size is set to be $\eta\le \frac{1}{c\kappa^2\sigma_1(X_\star)}$. When the condition number is large, the step size needs to be much smaller. In \cite{zhang2024fast}, the step size is set to be $\eta \le \min\left\{ \frac{L_{\delta}}{60\sqrt{2}(1+\delta+25(1+\delta)^2)},\frac{1}{7L_{\delta}} 
 \right\}$, which can easily be verified as a very small value. Therefore, APGD can converge with a larger step size, allowing it to converge faster than other methods.

\begin{remark}
\textbf{Comparison with NoisyPrecGD \cite{zhang2024fast} }
Similar to \cite{zhang2024fast}, we both consider the noisy matrix sensing problem and use preconditioning to accelerate the gradient descent. However, there are significant distinctions between our work and theirs, mainly in four aspects. First, both theoretically and experimentally, we prove that alternating update eliminate the need for a damping term, even in the presence of noise. This is a key difference from previous preconditioning-based methods, which emphasize the importance of balancing the damping parameter with the recovery error. Second, through alternating update, APGD is more robust to the step size and can converge more quickly with larger step sizes. Finally, NoisyPrecGD is limited to symmetric positive semi-definite matrices, which restricts its practical applicability. In contrast, our method is applicable to any matrix. 
\end{remark}

\subsection{Extension to general matrix estimation}
In this section, we further show that APGD can be applied to a wider range of low-rank matrix estimation problems, which can be modeled as
\begin{equation}
\underset{X\in\mathbb{R}^{n_1\times n_2}}{\operatorname{minimize} }\ g(X),\ \ \operatorname{s.t.} \ \operatorname{rank}(X)\le r.
\end{equation}

Based on the Burer–Monteiro (BM) factorization, the problem can be reformulated as the following optimization problem
\begin{equation}
\underset{L\in\mathbb{R}^{n_1\times r}, R \in\mathbb{R}^{n_2\times r}}{\operatorname{minimize} }\ g(LR^\top),
\label{equ:7}
\end{equation}
which is then solved using the APGD algorithm (Algorithm 2). In the initialization step of Algorithm 2, the specific spectral initialization method may vary depending on the problem and can be found in previous works \cite{tong2021accelerating,tu2016low}. However, spectral initialization is primarily a theoretical requirement. In practice, APGD can be directly initialized with random points, such as each entry of $L_0$ and $R_0$ is drawn independently from a random Gaussian distribution $\mathcal{N}(0, 1/n),\ n=\max\{n_1,n_2\}$.

\begin{algorithm}[h]
\caption{Alternating Preconditioned Gradient Descent (APGD) for low-rank matrix estimation}
\label{algorithm:2}
\textbf{Input:} Observations, step size $\eta$, estimated rank $r$.\\
{\textbf{Initialization}: Spectral initialization or random initialization}
\begin{algorithmic}[1] %[1] enables line numbers
\STATE \textbf{for} $t=0$ to $T-1$ \textbf{do}
\STATE \ \ \ \ \ $L_{t+1}=L_t-\eta\nabla_Lg(L_tR_t^\top)\cdot (R_t^\top R_t)^{\dagger}$
\STATE \ \ \ \ \ $R_{t+1}=R_t-\eta\nabla_R g(L_{t+1}R_t^\top)\cdot (L_{t+1}^\top L_{t+1})^\dagger$\\
 \ \ \ \ \  ($\dagger$ denotes the Moore-Penrose-Pseudo inverse)
\STATE \textbf{end for}
\STATE \textbf{return:} $X_T=L_TR_T^\top$
\end{algorithmic}
\end{algorithm}

To prove the convergence of APGD, we make some assumptions on the loss function \(g\), namely restricted smoothness and restricted strong convexity, which are commonly used in prior work \cite{zhang2023preconditioned,tong2021accelerating}.
\begin{definition} (Restricted smoothness, \cite{tong2021accelerating}) A differentiable function $g:\mathbb{R}^{n_1\times n_2} \mapsto \mathbb{R}$ is said to be rank-$r$ restricted $L_g$-smooth for some $L_g>0$ if 
$$
g(X_2) \le g(X_1) + \langle \nabla g(X_1), X_2 - X_1 \rangle + \frac{L_g}{2} ||X_1 - X_2||_F^2,
$$
for any $X_1,\ X_2\in\mathbb{R}^{n_1 \times n_2}$ with rank at most $r$.

\end{definition}

\begin{definition}(Restricted strong convexity, \cite{tong2021accelerating}) A differentiable function $g:\mathbb{R}^{n_1 \times n_2} \mapsto \mathbb{R}$ is said to be rank-$r$ restricted $\mu$-strongly convex for some $\mu\ge 0 $ if 
$$
g(X_2) \ge g(X_1) + \langle \nabla g(X_1), X_2 - X_1 \rangle + \frac{\mu}{2} || X_2 - X_1 ||_F^2,
$$
for any $X_1,\ X_2\in\mathbb{R}^{n_1 \times n_2}$ with rank at most $r$.

\end{definition}

Based on these two definitions, we present a new generalized theorem for APGD in the general case, with its proof provided in the Appendix \ref{proof of the general case}.
\begin{theorem}
Suppose that $g$ is rank-$2r$ restricted $L_g$-smooth and $\mu$-strongly convex, and $X_\star$ with rank-$r_\star$ denotes the minimizer, then if we have the initialization $X_0$ satisfies $||X_0-X_\star||_F \le \rho \sigma_r(X_\star),\ \rho\le \sqrt{\frac{3}{11}}$, and step size obeys $\eta\le 1/L_g$, then solving the low-rank matrix estimation problem (\ref{equ:7}) via APGD leads to
\begin{equation}
g(X_{t+1}) -g(X_\star) \le Q_g \left[ g(X_t) - g(X_\star) \right]
\notag
\end{equation}
where $Q_g =\left(1-\eta(1-\frac{L_g\eta}{2}) \zeta^2\right)^2$, $\zeta = \frac{(C_\rho-1)L_g +(C_\rho +1)\mu }{\sqrt{2L_g}}$, and $C_\rho = \sqrt{\frac{1-3\rho^2}{1-\rho^2}}$.
\label{theorem:general}
\end{theorem}

Based on this theorem, if we set $\rho = 0.1$ as in \cite{tong2021accelerating,zhang2024projected} and choose $\eta = 1/L_g$, then we have 
\begin{equation}
g(X_{t+1}) -g(X_\star) \le (1-0.198 \frac{\mu}{L_g})^2\left[ g(X_t) - g(X_\star) \right]
\label{equ:9}
\end{equation}
for $L_g/\mu\le 9801$.

\begin{remark}
As shown in previous works \cite{tong2021accelerating,zhu2021global,li2019non,zhu2018global}, many low-rank matrix estimation problems satisfy restricted smoothness and restricted strong convexity. For detailed proofs, please refer to \cite{li2019non}. Below, we list several tasks to which Theorem \ref{theorem:general} is applicable:
\begin{itemize}

    \item \textbf{Weighted matrix factorization} The loss function $g(X) = \frac{1}{2}||W \odot (X-X_\star)||_F^2$ satisfies rank-$2r$ restricted smooth with $L=\max W_{ij}^2$ and rank-$2r$ restricted strong convexity with $\mu=\min W_{ij}^2$.
    \item \textbf{Matrix Sensing}  The loss function $g(X) = \frac{1}{2} || \mathcal{A}(X-X_\star) ||_2^2$ satisfies rank-$2r$ restricted smooth with $L=1+\delta$ and rank-$2r$ restricted strong convexity with $\mu=1-\delta$ if the linear map $\mathcal{A}(\cdot)$ satisfies rank-$2r$ RIP with constant $\delta$.
    \item \textbf{Matrix completion} As proved in [\cite{jain2010guaranteed}, Theorem 4.2], When the sampling rate exceeds a certain threshold, all rank-\(r\) matrices that are \(\xi\)-incoherent satisfy the rank-\(r\) RIP condition. Here, a matrix \(X \in \mathbb{R}^{n_1 \times n_2}\) with singular value decomposition $X = USV^\top$ is said to be $\xi$-incoherent if it satisfies  $$\max_{ij} |U_{ij}| \le \sqrt{\frac{\xi}{n_1}}, \quad \max_{ij} |V_{ij}| \le \sqrt{\frac{\xi}{n_2}}.$$ Therefore, under certain conditions, matrix completion can be viewed as a special case of matrix sensing, and thus naturally satisfies the rank-$2r$ restricted smoothness and restricted strong convexity.
    
\end{itemize}
\end{remark}

\begin{table}[ht]
\caption{Comparison of related works in low-rank matrix estimation. In the second column, the upper bounds of step size in the previous work are listed. The third line refers to the decay rate of the loss function $g$, defined as $Q_g = \frac{g(X_{t+1}) - g(X_\star)}{g(X_t) - g(X_\star)}.
$ The fourth column indicates whether the asymmetric factorization is considered. The fifth column indicates whether the over-rank situation is considered}
\centering
\setlength{\tabcolsep}{4pt} 
\begin{tabular}{cccccc}
\hline
methods & step size  & decay rate $Q_g$ & asymmetry & over rank\\ \hline

ScaledGD \cite{tong2021accelerating} & $\le \frac{2}{5L_g}$   &  $1- \frac{7\mu}{25L_g}$ & \Checkmark  & \XSolidBrush \\ \hline
PrecGD \cite{zhang2023preconditioned} & $= \frac{1}{4L_g}$   & $1-\frac{\mu^2}{8L_g^2}$ & \XSolidBrush & \Checkmark \\ \hline
ProjGD \cite{zhang2024projected} &  $\le \frac{1}{2L_g}$   & $1-\frac{\mu}{27L_g}$ & \Checkmark & \Checkmark\\ \hline

ours &  $\le\frac{1}{L_g}$   & $(1-\frac{0.198\mu}{L_g})^2$ & \Checkmark  & \Checkmark \\ \hline

\end{tabular}
\label{table:1}
\end{table}

\begin{remark} The works \cite{tong2021accelerating,zhang2023preconditioned} have also investigated general low-rank matrix estimation problems via preconditioning technique. \cite{tong2021accelerating} analyzes the convergence of ScaledGD under restricted smoothness and restricted strong convexity. Compared to \cite{tong2021accelerating}, APGD can handle over-parameterized settings. \cite{zhang2023preconditioned} studies the convergence of PrecGD in the over-parameterized case, but it requires estimating a damping parameter and is limited to symmetric positive semidefinite matrices. In contrast, APGD does not require tuning a damping parameter and can be applied to general (not necessarily symmetric or PSD) matrices. 
\end{remark}

\begin{remark}
A recent advancement in low-rank matrix estimation is the Projected Gradient Descent (ProjGD) method introduced by Zhang et al. \cite{zhang2024projected}. They established both local and global convergence guarantees for ProjGD and demonstrated that it exhibits a linear convergence rate. Specifically, they proved that if the initialization satisfies $\|X_0 - X_\star\|_F \le 0.1 \sigma_{r_\star}(X_\star)$ and the step size obeys $\eta \le 1/(2L_g)$, then

$$
g(X_{t+1}) - g(X_\star) \le \left(1 - \frac{\mu}{27L_g} \right) [g(X_t) - g(X_\star)].
$$

In contrast, APGD method achieves a faster convergence rate, as shown in Table 2, and demonstrates greater robustness with respect to the choice of step size. Moreover, ProjGD requires computing the SVD at each iteration, which is computationally expensive. As a result, its actual runtime increases rapidly with the matrix size. In contrast, APGD is significantly more practical and scalable in real applications.

\end{remark}

\section{Key idea and proof sketch}
% In this section, we first demonstrate the role of alternating update, which avoids the damping term in the preconditioner and enhances the algorithm's robustness to the step size. We then outline the steps to prove the global convergence of APGD, with the detailed proof provided in the appendix \ref{sec:proof}.

\subsection{The role of damping parameter $\lambda$ in previous works}
First, we examine why previous works \cite{zhang2021preconditioned,xu2023power,zhang2024fast} rely on the damping term \( \lambda I \). To address the slow convergence of gradient descent in the over-parameterized and ill-conditioned cases, \cite{zhang2021preconditioned} introduced PrecGD, which accelerates convergence by adding a right preconditioner after the gradient. Based on the preconditioner \( P =  L^\top L + \lambda I \), they defined the corresponding local P-norm:
\begin{equation}
    \|X\|_P\overset{\operatorname{def}}{=}\|XP^{\frac{1}{2}}\|_F,\ \|X\|_P^*\overset{\operatorname{def}}{=}\|XP^{-\frac{1}{2}}\|_F.
\end{equation}
Using this, they derived an inequality similar to a Lipschitz condition: 
\begin{equation}
f(L-\eta D) \le f(L) -\eta \langle \nabla f(L), D \rangle +\frac{\eta^2 L_p}{2}\| D \|_P^2, 
\label{equ:8} 
\end{equation} where
$$
L_p = 2(1+\delta)\left[ 4+ \frac{2\|E_\natural\|_F+4\|D\|_P}{\lambda_r^2(L)+\lambda} + \left( \frac{\|D\|_P}{\lambda_r^2(L)+\lambda} \right)^2 \right],
$$
$D$ is the descent direction, and for simplicity, $L L^\top - X_\star  = E_\natural $.

From the above inequality, we can observe that the smaller \( L_p \) is, the faster the algorithm converges. Moreover, from the definition of \( L_p \), we can see that the smaller \( L_p \) becomes, the larger \( \lambda \) must be. However, the convergence of the algorithm also depends on another inequality, namely the Polyak-Lojasiewicz inequality:
\begin{equation}
 \langle \nabla f(L), D \rangle \overset{(i)}{=} \|\nabla f(L)\|_P^* \ge \mu_P f(L),
\end{equation}
where $(i)$ using the assumption that $D=\nabla f(L)(L^\top L +\lambda I)^{-1}$. 
From this inequality, we see that larger \( \mu_P \) leads to faster the convergence. However, \cite{zhang2021preconditioned} proved that as \( \mu_P \) increases, \( \lambda \) must decrease. Combining these two inequalities, for PrecGD, \( \lambda \) must satisfy $\lambda_t = \Theta(\|L_t^\top L_t - X_\star\|_F).$

Next, let's analyze Equation (\ref{equ:8}) in detail to understand why \( L_p \) is related to \( \lambda \). We will derive Equation (\ref{equ:8}) step by step to understand this relationship. 

Let us proceed with the detailed derivation:
\begin{equation}
\begin{aligned}
& f(L-\eta D) = \left\| \mathcal{A}\left((L-\eta D)(L-\eta D)^\top-X_\star \right) \right\|_2^2 \\ 
&=   \|\A(E_\natural)\|_2^2 -  2 \langle \A(E_\natural), \A(LD^\top + DL^\top) \rangle \\
& + \| \A(LD^\top+DL^\top)\|_2^2 +  \langle \A(LD^\top + DL^\top), \A(DD^\top)\rangle \\
&-2\langle \A(E_\natural), \A(DD^\top)  \rangle + \| \A(DD^\top) \|_2^2 .
\end{aligned}
\notag
\end{equation}
From this expression, we can see that the quadratic term of the gradient, \( DD^\top \), is the term that makes \( L_p \) related to the damping parameter \( \lambda \). For example, for \( \A(DD^\top) \), we have:
\begin{equation}
\|\A(DD^\top)\|_2^2 \leq (1 + \delta)^2 \|D\|_F^4 \leq \frac{\|D\|_P^4}{\lambda_r^2(L) + \lambda}.
\label{equ:12}
\end{equation}

This shows that \( L_p \) becomes dependent on \( \lambda \) as the damping parameter influences the magnitude of the quadratic gradient term.

\subsection{How alternating helps: damping free and large step size}

As shown in Equation (\ref{equ:12}), the quadratic term of the gradient \( D \) is the reason why \( L_p \) depends on \( \lambda \). It is important to note that a similar issue arises for the non-symmetric decomposition \( X = L R^\top \), since GD synchronously updates the two factor matrices $L$ and $R$. Therefore, if we can avoid this term, then \( L_p \) would no longer depend on \( \lambda \).  Unlike GD, APGD updates the two factor matrices in an alternating manner, which avoids the quadratic terms in the gradient.

Based on Algorithm 1, we can derive the following Lemma for the noiseless case, 
\begin{lemma}
For the noiseless matrix sensing problem, suppose that the linear map $\A(\cdot)$ satisfies the rank-($2r+1$) RIP with constant $\delta_{2r+1}$, then we have
\begin{equation}
\begin{aligned}
f_c(L_t -\eta D^L_t,R_t) & \le f(L_t,R_t) 
- \eta \langle  \nabla_L f(L_t,R_t), D^L_t\rangle \\
& + \frac{\eta^2L_f}{2} \| D^L_t (R_t^\top R_t)^{\frac{1}{2}} \|_F^2 \\
f_c(L_{t+1},R_t-\eta D^R_t) & \le f(L_{t+1},R_t)  - \eta \langle  \nabla_R f(L_{t+1},R_t), D^R_t\rangle  \\
&  + \frac{\eta^2L_f}{2} \| D^R_t (L_{t+1}^\top L_{t+1})^{\frac{1}{2}} \|_F^2,
\end{aligned}
\notag
\end{equation}
where $D^L_t = \nabla_L f(L_t,R_t)(R_t^\top R_t)^\dagger$ and $D^R_t = \nabla_R f(L_{t+1},R_t)(L_{t+1}^\top L_{t+1})^\dagger $ are the descent directions of APGD and $L_f = 1+\delta_{2r+1}$.
\label{Lemma: noiseless lipschitz}
\end{lemma}
\begin{IEEEproof}
See Appendix \ref{Proof of Lipschitz-like inequality}.
\end{IEEEproof}

From this lemma, we can see that for APGD, \( L_f \) is independent of the damping parameter. In other words, APGD does not require a damping parameter. This is one of the key advantages of APGD, as it avoids the need for careful tuning of the damping parameter, which is typically required in methods like PrecGD.  

Another advantage of APGD is its robustness to the step size. As is well known, the upper bound on the step size in gradient descent depends on the gradient Lipschitz constant \( L \), i.e., \( \eta \leq \frac{1}{L} \). For other preconditioned methods, the value of \( L \) is typically very large, which results in a very small step size, as discussed in Section 3. However, for APGD, the step size only needs to satisfy \( \eta \leq \frac{1}{1 + \delta_{r + r_\star}} \), which is a rather mild condition.

\subsection{Proof outline}
Based on the above analysis, we outline the proof of APGD convergence under noisy conditions. First, inspired by the work of \cite{zhang2021preconditioned,zhang2024fast} and \cite{cheng2024accelerating}, we introduce two local norms and their corresponding dual norms
\begin{equation}
\begin{aligned}
    &\|A\|_{R_t} \overset{\operatorname{def}}{=}\|AP_{R_t}^{\frac{1}{2}}\|_F,\ \|A\|^*_{R_t} \overset{\operatorname{def}}{=}\|AP_{R_t}^{\frac{\dagger}{2}}\|_F,\ P_{R_t}\overset{\operatorname{def}}{=} R_t^\top R_t, \\
    &\|A\|_{L_t} \overset{\operatorname{def}}{=}\|AP_{L_t}^{\frac{1}{2}}\|_F,\ \|A\|^*_{L_t} \overset{\operatorname{def}}{=}\|AP_{L_t}^{\frac{\dagger}{2}}\|_F,\ P_{L_t}\overset{\operatorname{def}}{=} L_t^\top L_t.
\end{aligned}
\notag
\end{equation}
 Using these norms, we derive a Lipschitz-like lemma.
 \begin{lemma}(Lipschitz-like inequality)
Suppose that we have $ \| \nabla_L f_c(L_t,R_t) \|_{P_{R_t}^*} \ge 3 \| \mathcal{A}^*(s)R_t \|_{P_{R_t}^*}$, $\|\nabla_R f_c(L_{t+1},R_t)\|_{P_{L_{t+1}}^*} \ge 3 \| \mathcal{A}^*(s)L_{t+1}^\top \|_{P_{L_{t+1}}^*}$, and $\A(\cdot)$ satisfies the rank-$(2r+1)$ RIP with constant $\delta_{2r+1}$, then we have  
\begin{equation}
\begin{aligned}
& f_c (L_{t+1},R_t) \le f_c(L_t,R_t) - C_2 \| \nabla_L f_c(L_t,R_t) \|_{P_{R_t}^*}\\ 
 &f_c (L_{t+1},R_{t+1}) \le f_c(L_{t+1},R_t) - C_2 \| \nabla_R f_c(L_{t+1},R_t) \|_{P_{L_{t+1}}^*}\\
\end{aligned}
\notag
\end{equation}
where $C_2=\eta - \frac{\eta}{3}(1 + 2 \eta (1+\delta_{2r+1}) )$.
\label{lemma:lipschitz}
\end{lemma}
\begin{IEEEproof}
See Appendix \ref{Proof of Lipschitz-like inequality}.
\end{IEEEproof}

 The key difference between this lemma and the previous noise-free lemma is the inclusion of assumptions on the noise term $\{\| \mathcal{A}^*(s)R_t \|_{P_{R_t}^*}, \| \mathcal{A}^*(s)L_{t+1}^\top \|_{P_{L_{t+1}}^*}\}$ and the gradient term $\{\| \nabla_L f_c(L_t,R_t) \|_{P_{R_t}^*},\|\nabla_R f_c(L_{t+1},R_t)\|_{P_{L_{t+1}}^*}\}$. This new lemma demonstrates that when the gradient term dominates the noise term, APGD converges linearly. 

Next, we need to establish a lower bound for the gradient term, which leads to the following lemma.

\begin{lemma}
Suppose that the linear map $\A(\cdot)$ satisfy the $\delta_{2r+1}$-RIP and the initial point $||L_0R_0^\top -X_\star||_F\le \rho \sigma_{r_\star}(X_\star),\ \rho\le \frac{1}{2}$, then we have
\begin{equation}
\begin{aligned}
\|\nabla_L f_c (L_t,R_t)\|_{P_{R_t}^*}^2 &\ge \tau f_c (L_{t},R_t), \\
\|\nabla_R f_c (L_{t+1},R_t)\|_{P_{L_{t+1}}^*}^2 &\ge \tau f_c (L_{t+1},R_t), \\
\end{aligned}
\end{equation}
where $\tau = \left( \sqrt{\frac{1-3\rho^2}{1-\rho^2}} -\sqrt{r+r_\star} \delta_{2r+1} \right)^2.$
\label{lemma: gradient dominance}
\end{lemma}
\begin{IEEEproof}
See Appendix \ref{Proof of the gradient dominance}.
\end{IEEEproof}
Combining these two lemmas, we can easily conclude that when the gradient term dominates the noise term, APGD converges linearly, i.e., 
\begin{equation}
\begin{aligned}
f_c(L_{t+1},R_{t+1})& \le Q_f^2 \cdot f_c(L_t,R_t)\\
\|L_tR_t^\top-X_\star\|_F^2 &\le  C_\delta Q_f^{2t} \cdot \|L_0 R_0^\top - X_\star\|_F^2,
\end{aligned}
\end{equation}
where $Q_f$ and $C_\delta$ are the same parameters as defined in Theorem \ref{main theorem}.

Next, we need to consider the case where the noise term is smaller than the gradient term. In this case, we can combine Lemma \ref{lemma: gradient dominance} to derive 
\begin{equation}
\begin{aligned}
     &f_c (L_{t},R_t) \le \frac{1}{\tau} \|\nabla_L f_c (L_t,R_t)\|_{P_{R_t}^*}^2, \\
 & f_c (L_{t+1},R_t) \le \frac{1}{\tau} \|\nabla_R f_c (L_{t+1},R_t)\|_{P_{L_{t+1}}^*}^2. \\
\end{aligned}
\label{equ:13}
\end{equation}

Then, combining equation (\ref{equ:13}) and matrix concentration bounds, we can conclude that when the gradient term is smaller than the noise term, we have 
$$
\| L_tR_t^\top - X_\star \|_F^2 \le C_3 \cdot \mathcal{E}_{opt}.
$$
This is the general outline of the proof for Theorem 3. The detailed proof can be found in Appendix \ref{proof of the main results}.

\section{Experiments}
In this section, we conduct a series of experiments to evaluate the effectiveness of APGD. Results on the noisy matrix sensing task show that APGD does not require an additional damping parameter and is highly robust to the choice of step size. It achieves linear convergence to near-minimal error even in over-parameterized and ill-conditioned settings. Compared to NoisyPrecGD \cite{zhang2024fast} and GD \cite{ding2022validation}, APGD requires fewer iterations and less computation time. In addition, we conduct both synthetic and real-data experiments on other low-rank matrix estimation tasks, including weighted PCA \cite{srebro2003weighted}, 1-bit matrix completion \cite{davenport20141}, and matrix completion \cite{jain2010guaranteed}. The results demonstrate that APGD can be broadly applied to a wide range of low-rank matrix estimation problems. The experimental code is available at \href{https://github.com/ZhiyuLiu3449/APGD}{https://github.com/ZhiyuLiu3449/APGD}.

\subsection{Experiments for noisy matrix sensing}
\textbf{Experimental setup} 
The target rank-$r_\star$ matrix $X_\star\in\mathbb{R}^{n_1\times n_2}$ with condition number $\kappa$ is generated as $X_\star=U_\star \Sigma V_\star^\top$, where $U_\star$ and $V_\star$ are both orthogonal matrix and $\Sigma$ is a diagonal matrix with condition number $\kappa$. The entries of the sensing matrix $A_i$ are sampled i.i.d from distribution $\mathcal{N}(0,\frac{1}{m})$. The entries of the noise $\textbf{s}$ are sampled i.i.d from distribution $\mathcal{N}(0,\nu^2)$. For all three methods, we adopt the spectral initialization described in Algorithm 1.

\begin{figure*}[htbp]
\centering
\subfigure[]{
\begin{minipage}[t]{0.25\linewidth}
\centering
\includegraphics[width=4.5cm,height=4.5cm]{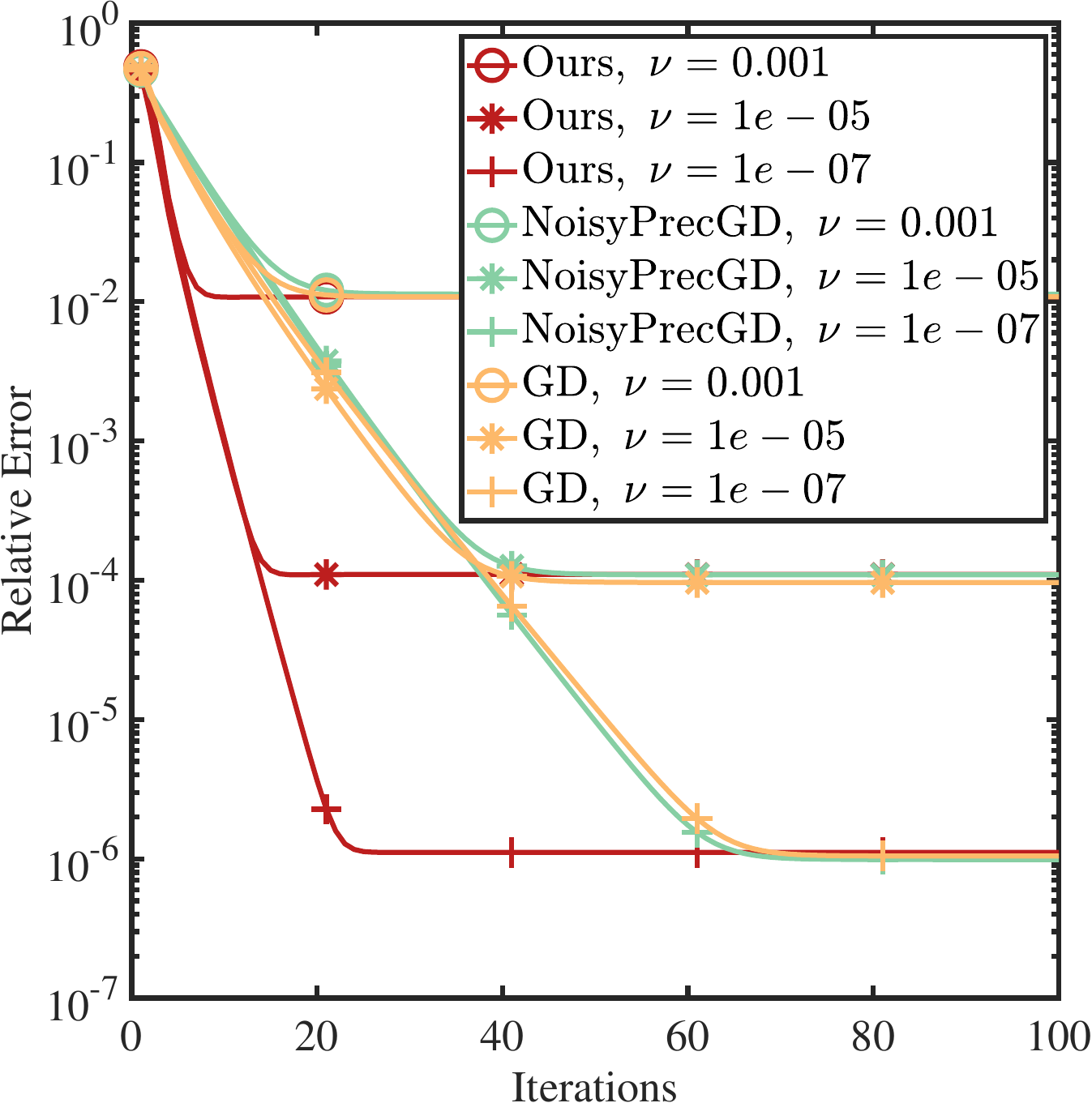}
\end{minipage}%
}%
\subfigure[]{
\begin{minipage}[t]{0.25\linewidth}
\centering
\includegraphics[width=4.5cm,height=4.5cm]{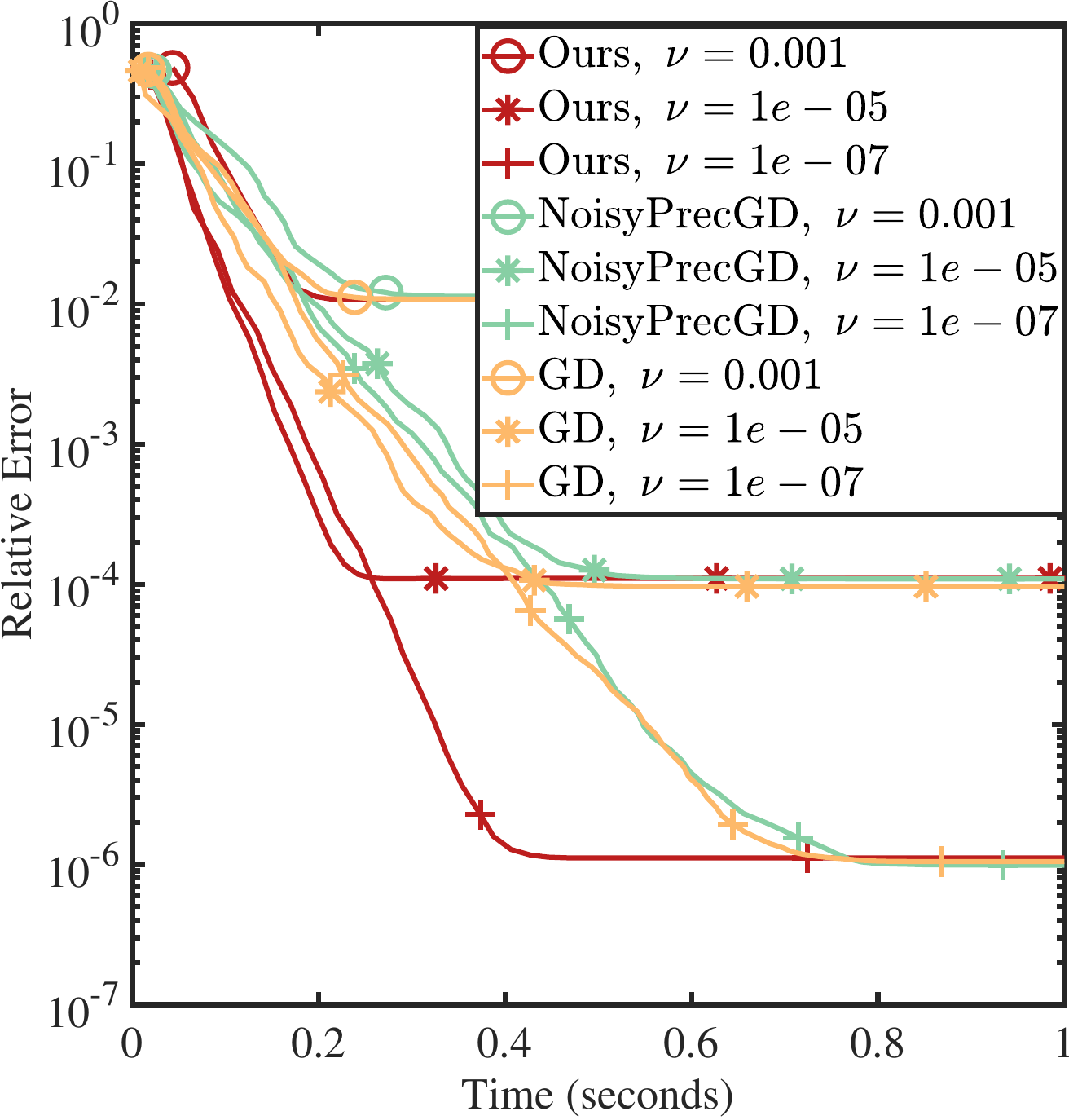}
\end{minipage}%
}%
\subfigure[]{
\begin{minipage}[t]{0.25\linewidth}
\centering
\includegraphics[width=4.5cm,height=4.5cm]{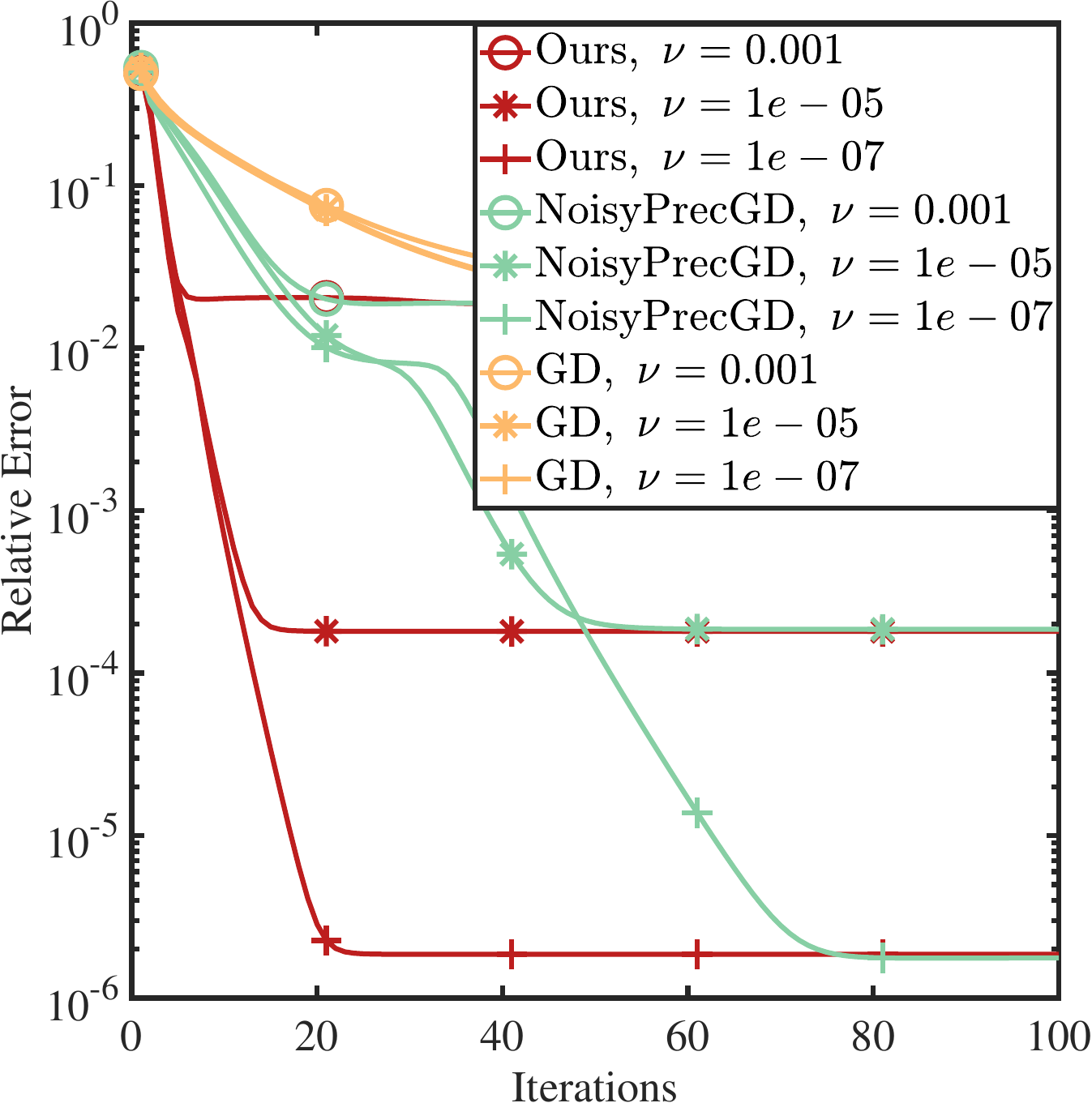}
\end{minipage}%
}%
\subfigure[]{
\begin{minipage}[t]{0.25\linewidth}
\centering
\includegraphics[width=4.5cm,height=4.5cm]{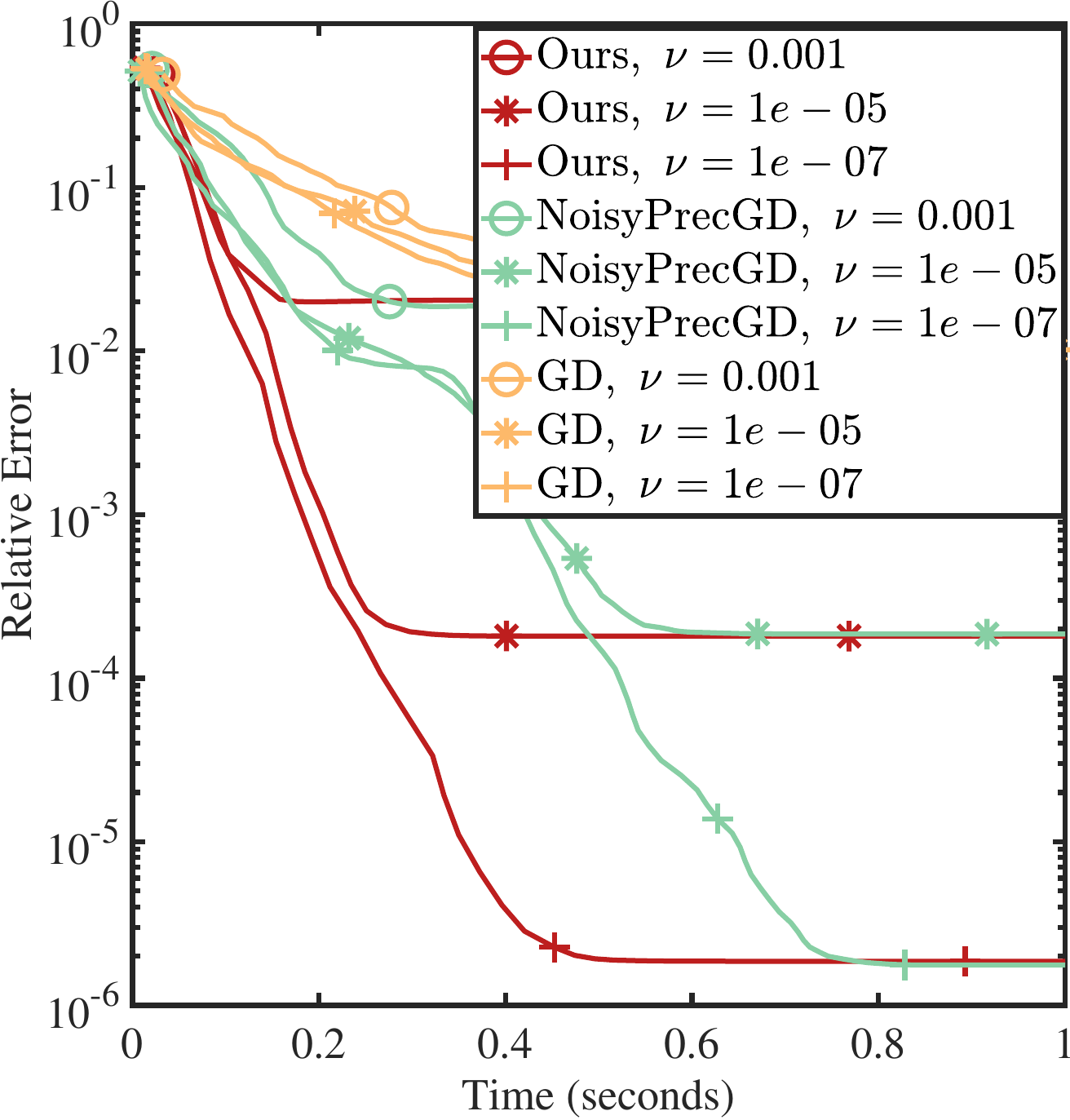}
\end{minipage}%
}%
\centering

\caption{Relative recovery error and computation time of NoisyPrecGD, GD, and APGD on the exact-rank noisy matrix sensing problem, where $n_1 = n_2 = 20$, $r_\star = r = 5$, and $m = 10n_1r$. The step sizes for each method are tuned to achieve the fastest convergence: APGD uses a step size of 1, while GD uses a step size of 0.5 and NoisyPrecGD uses a step size of 0.7. Subfigures (a) and (b) correspond to a condition number of 1, while (c) and (d) correspond to a condition number of 100.
}
\label{fig:1}
\end{figure*}

\begin{figure*}[htbp]
\centering
\subfigure[]{
\begin{minipage}[t]{0.25\linewidth}
\centering
\includegraphics[width=4.5cm,height=4.5cm]{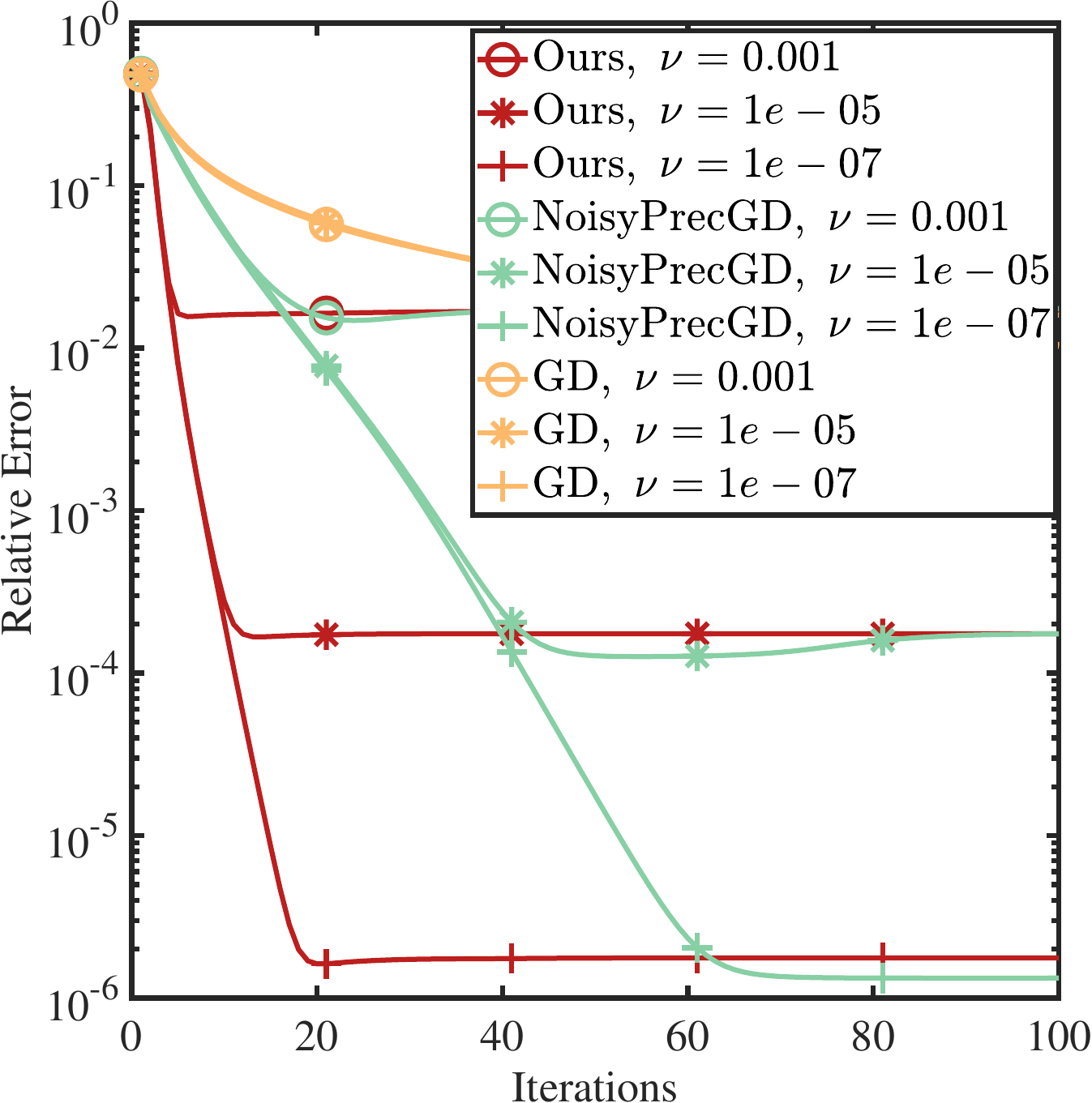}
\end{minipage}%
}%
\subfigure[]{
\begin{minipage}[t]{0.25\linewidth}
\centering
\includegraphics[width=4.5cm,height=4.5cm]{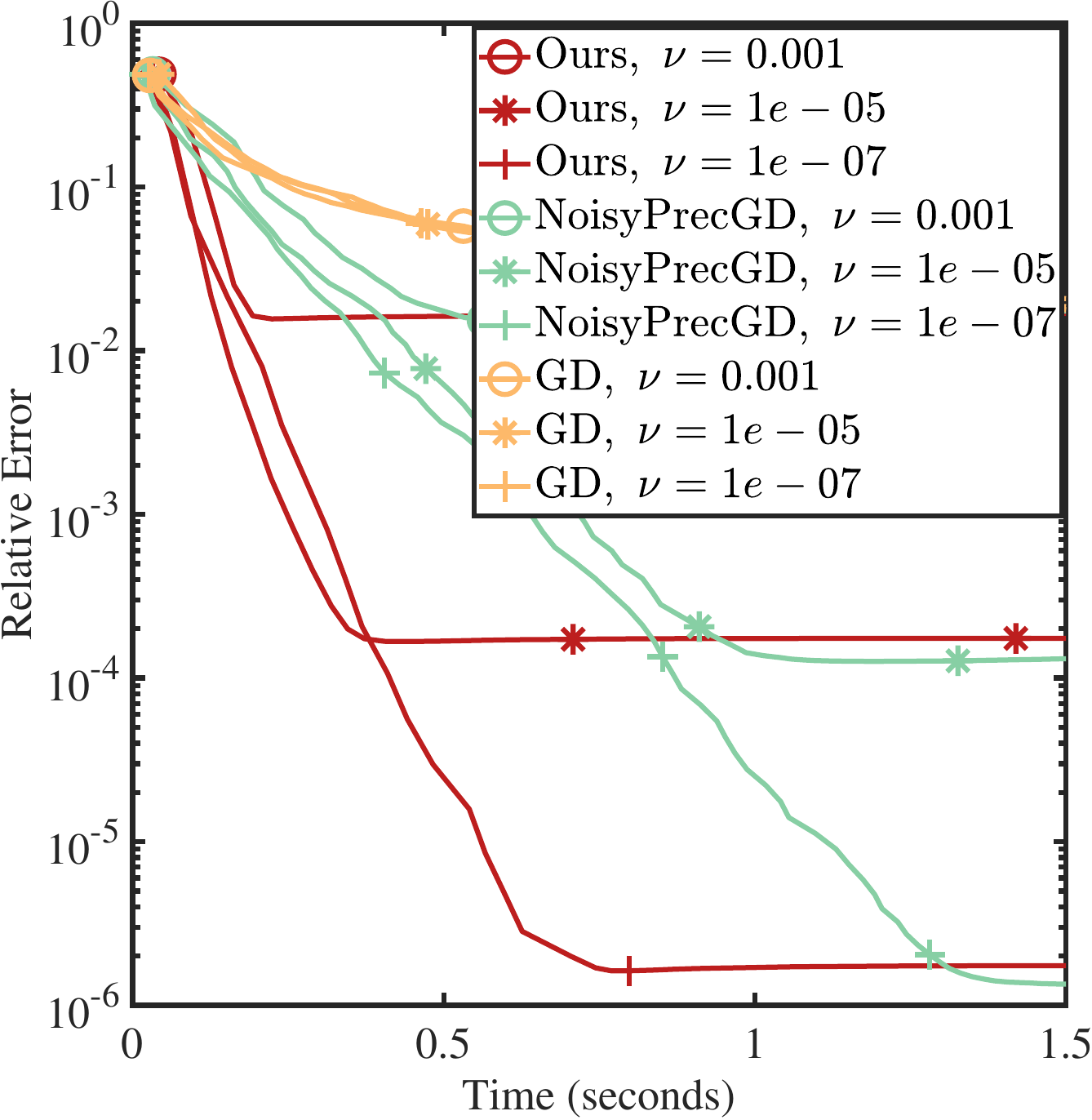}
\end{minipage}%
}%
\subfigure[]{
\begin{minipage}[t]{0.25\linewidth}
\centering
\includegraphics[width=4.5cm,height=4.5cm]{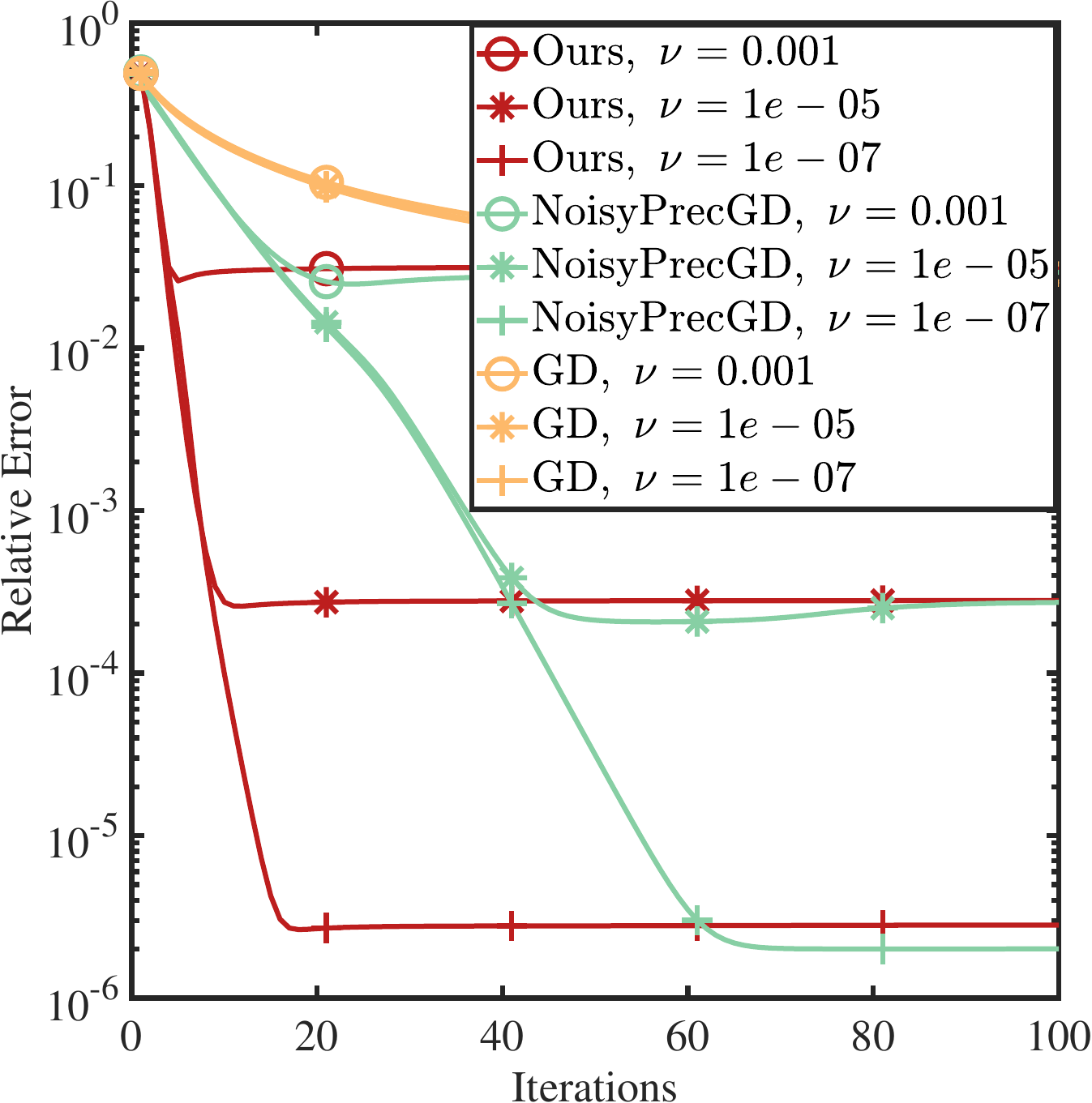}
\end{minipage}%
}%
\subfigure[]{
\begin{minipage}[t]{0.25\linewidth}
\centering
\includegraphics[width=4.5cm,height=4.5cm]{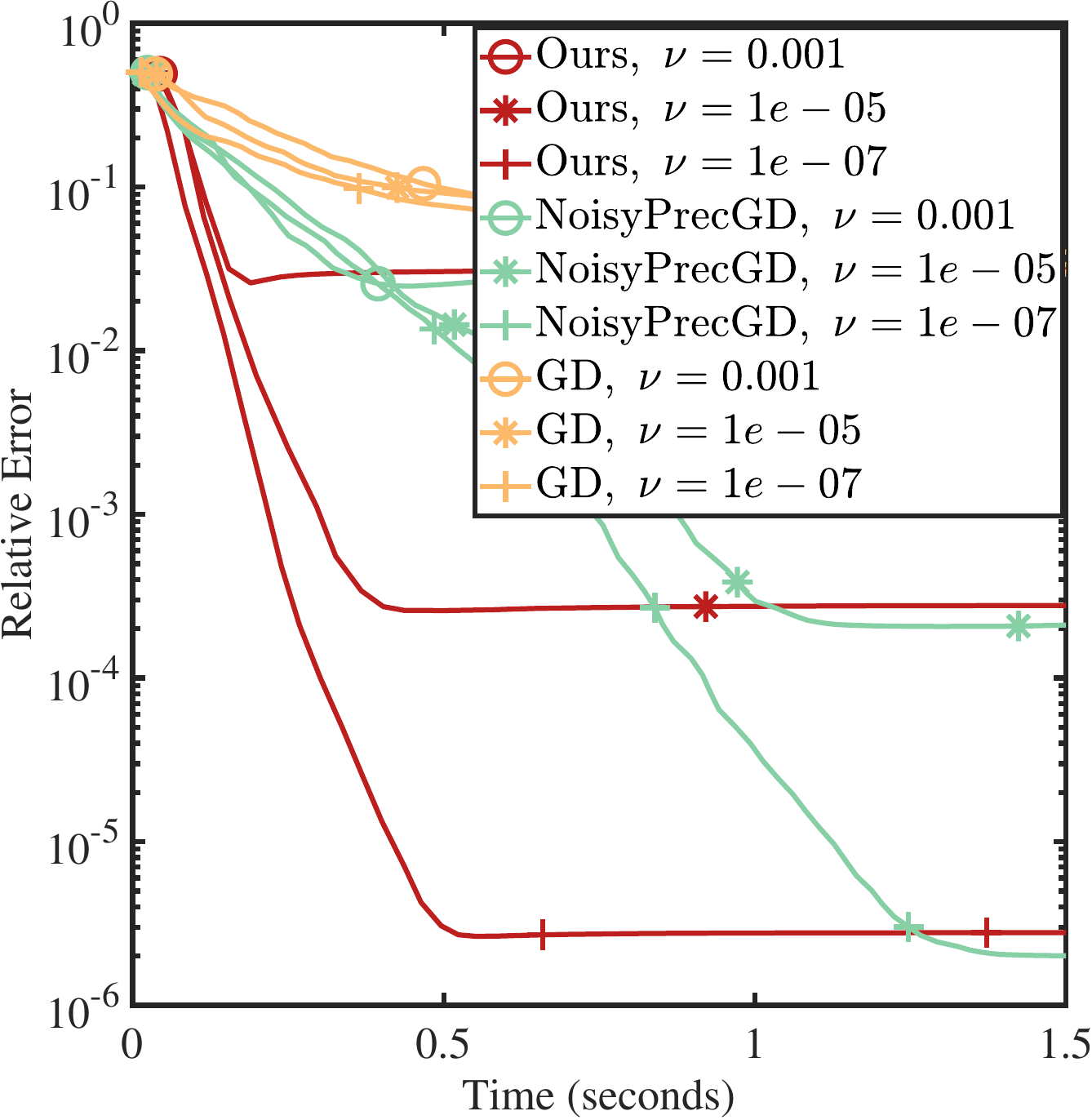}
\end{minipage}%
}%
\centering

\caption{Relative recovery error and computation time of NoisyPrecGD, GD, and APGD on the over-rank noisy matrix sensing problem, where $n_1 = n_2 = 20$, $r_\star =5,\ r=2r_\star$, and $m = 10n_1r$. The step sizes for each method are tuned to achieve the fastest convergence: APGD uses a step size of 1, while GD uses a step size of 0.5 and NoisyPrecGD uses a step size of 0.7. Subfigures (a) and (b) correspond to a condition number of 1, while (c) and (d) correspond to a condition number of 100.
}
\label{fig:2}
\end{figure*}

\textbf{Comparison with GD and NoisyPrecGD} 
Figures \ref{fig:1} and \ref{fig:2} show the relative recovery error and computation time of different methods under varying ranks $r$ and condition numbers $\kappa$. Compared to NoisyPrecGD and GD, APGD exhibits a significantly faster convergence rate. Although each iteration of APGD involves recomputing the gradient and thus incurs a higher per-iteration cost, its overall computation time is still lower than that of the other two methods. Moreover, both NoisyPrecGD and APGD are unaffected by the condition number and over-parameterization, whereas GD is sensitive to both, highlighting the effectiveness of preconditioning.

\textbf{Evaluating the robustness of step size} We evaluate the robustness of the three methods to step size in the noiseless setting and show that APGD exhibits the strongest robustness. As shown in Figure \ref{fig:3}, when the step size is small, APGD and PrecGD perform similarly; however, as the step size increases, APGD converges faster, while PrecGD and GD diverge when the step size exceeds 0.8.

\textbf{Comparison with \cite{ding2022validation}} 
In Figure \ref{fig:4}, we compare APGD with GD using small random initialization, as \cite{ding2022validation} demonstrated that GD with small random initialization can converge to the optimal error. As shown in Figure \ref{fig:4}, in the exact-rank setting, APGD and GD with small initialization achieve similar recovery errors. In the over-parameterized case, GD yields slightly lower recovery error than APGD. However, as noted in previous work \cite{zhang2024fast}, when $r = \mathcal{O}(r_\star)$, the recovery errors of both methods can be considered of the same order. Moreover, GD requires 100 times more iterations than APGD.
Therefore, APGD is more practical due to its faster convergence and tolerable recovery error.

\subsection{Experiments for more general cases}
\label{extension experiments}

\subsubsection{Weighted low-rank matrix factorization} 
The weighted PCA problem is defined as recovering the rank-$r_\star$ matrix $X_\star\in\mathbb{R}^{n_1 \times n_2}$ from the observation $O=W \odot X_\star$, where $W$ denotes the knowing weight matrix. We can solve this problem by minimizing the following objective function using Burer–Monteiro factorization:
\begin{equation}
\underset{L \in \mathbb{R}^{n_1 \times r},\ R \in \mathbb{R}^{n_2 \times r}}{\operatorname{minimize}} \frac{1}{2} ||  W\odot (LR^\top - X_\star) ||_F^2.
\notag 
\end{equation}
As shown in \cite{li2019non}, when the condition $\frac{\max W_{ij}^2}{\min W_{ij}^2 } \le 1.5$ holds, the objective function has no spurious local minima. In this experiment, we relaxed the condition and generated weight matrices with $\frac{\max W_{ij}^2}{\min W_{ij}^2} = 4$ for our simulation experiments. As shown in Figure \ref{fig:5}, under different condition numbers, APGD demonstrates faster convergence rates and shorter computation times compared to the other two methods.

\begin{figure}[h]
\centering
\subfigure[]{
\begin{minipage}[t]{0.48\linewidth}
\centering
\includegraphics[width=4.3cm,height=4.3cm]{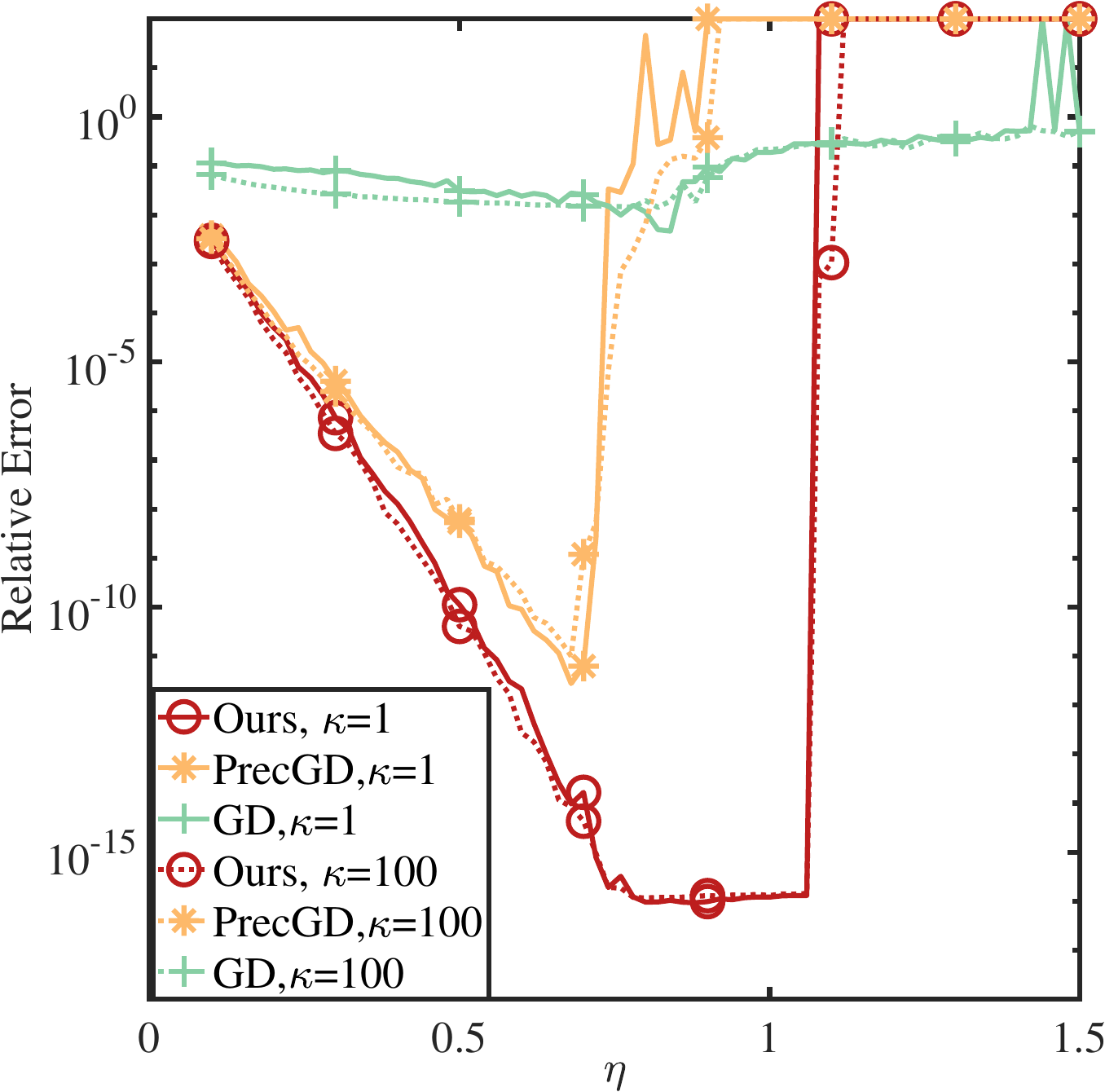}
\end{minipage}%
}
\subfigure[]{
\begin{minipage}[t]{0.48\linewidth}
\centering
\includegraphics[width=4.3cm,height=4.3cm]{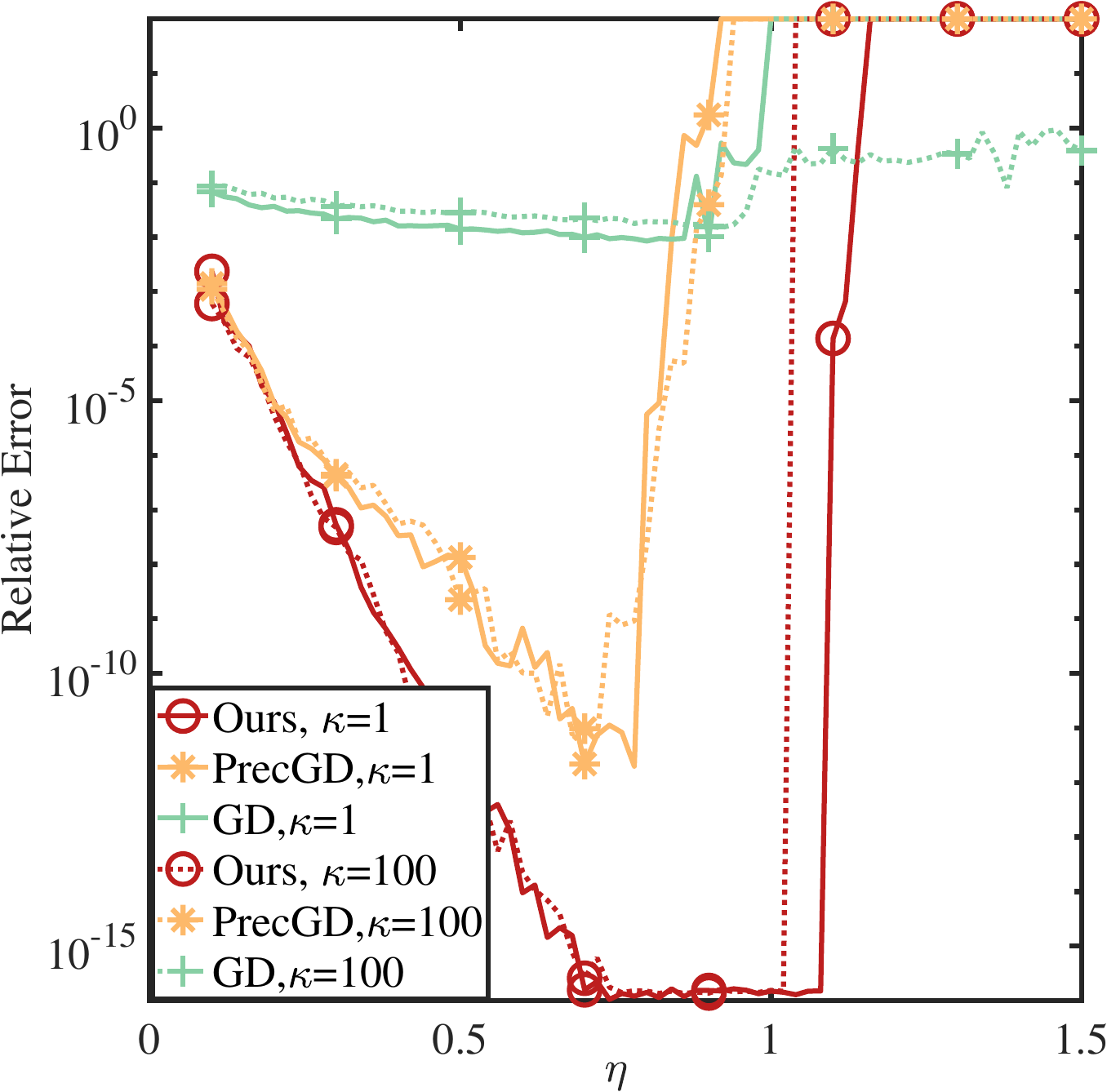}
\end{minipage}%
}
\centering
\caption{The relative error of APGD, PrecGD, and GD after 100 iterations with respect to different step size $\eta$ under different condition numbers for matrix sensing. $n=20,\ r_\star =5,\ m=10nr$. Subfigure (a) denotes the exact rank case with $r=r_\star$ while subfigure (b) denotes the over-rank case with $r=2r_\star$.  }
\label{fig:3}
\end{figure}

\begin{figure}[h]
\centering
\subfigure[]{
\begin{minipage}[t]{0.48\linewidth}
\centering
\includegraphics[width=4.3cm,height=4.3cm]{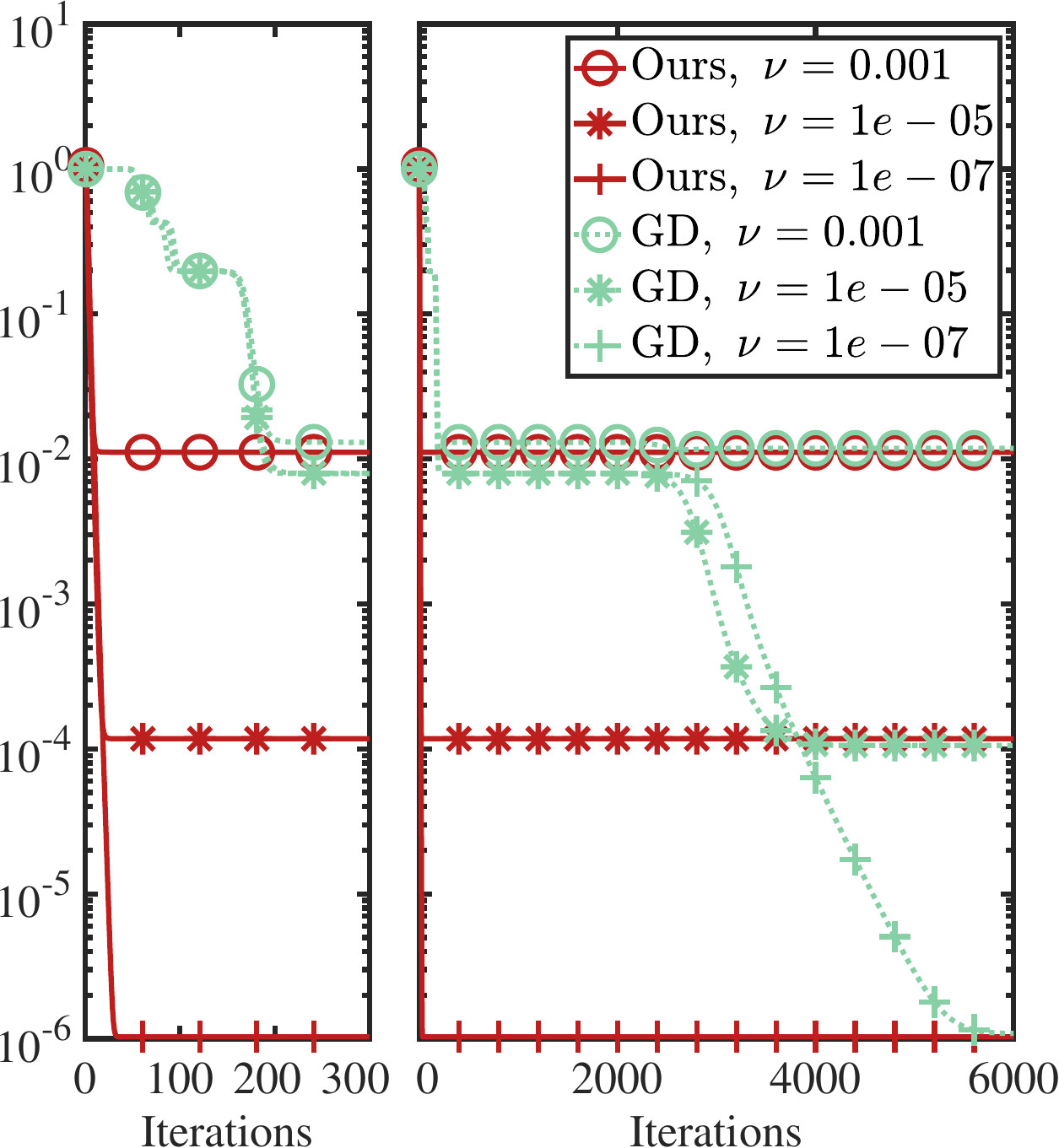}
\end{minipage}%
}
\subfigure[]{
\begin{minipage}[t]{0.48\linewidth}
\centering
\includegraphics[width=4.3cm,height=4.3cm]{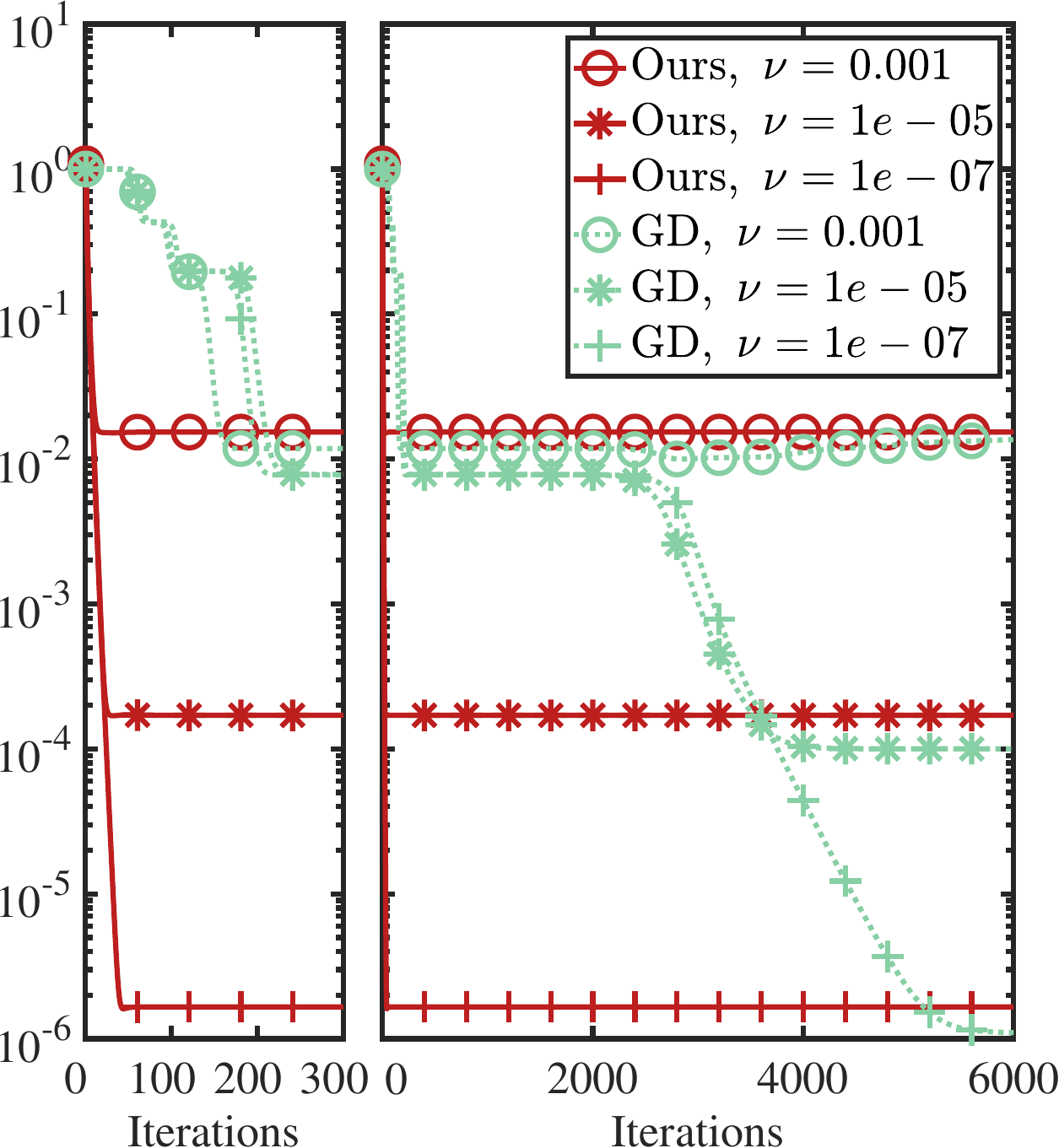}
\end{minipage}%
}
\centering
\caption{Recovery error of APGD (with spectral initialization) and GD with small initialization under different noise levels, where $n_1 = n_2 = 20$, $r_\star = 5$, $m = 10n_1r$, and $\kappa = 100$. The step size for APGD is 1, and for GD is 0.5. Subfigure (a) corresponds to the exact-rank setting, while subfigure (b) shows the over-parameterized case. In each subfigure, the left plot shows the first 300 iterations, and the right plot shows all 6000 iterations.
  }
\label{fig:4}
\end{figure}

\begin{figure}[h]
\centering
\subfigure[]{
\begin{minipage}[t]{0.48\linewidth}
\centering
\includegraphics[width=4.3cm,height=4.3cm]{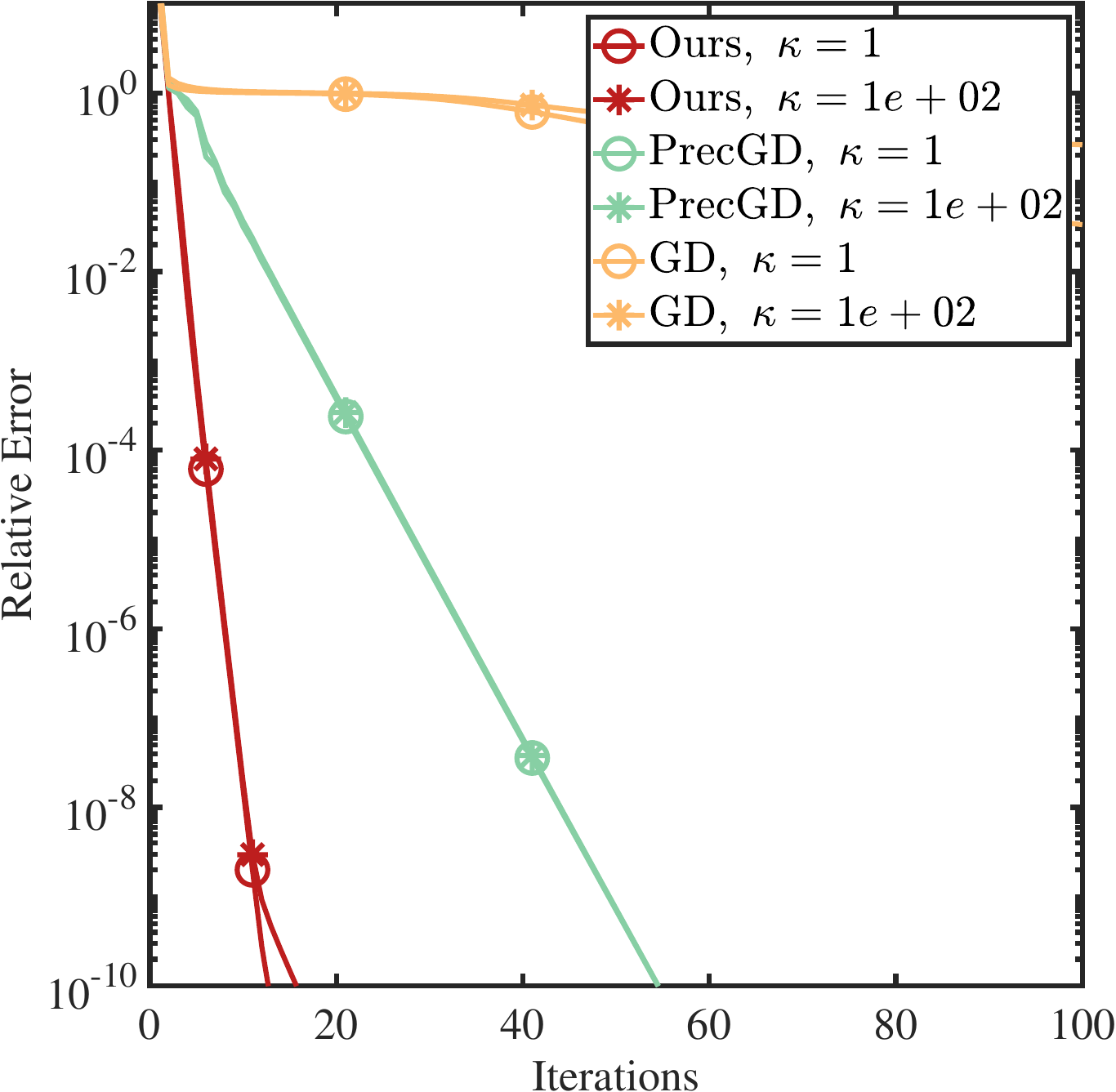}
\end{minipage}%
}
\subfigure[]{
\begin{minipage}[t]{0.48\linewidth}
\centering
\includegraphics[width=4.3cm,height=4.3cm]{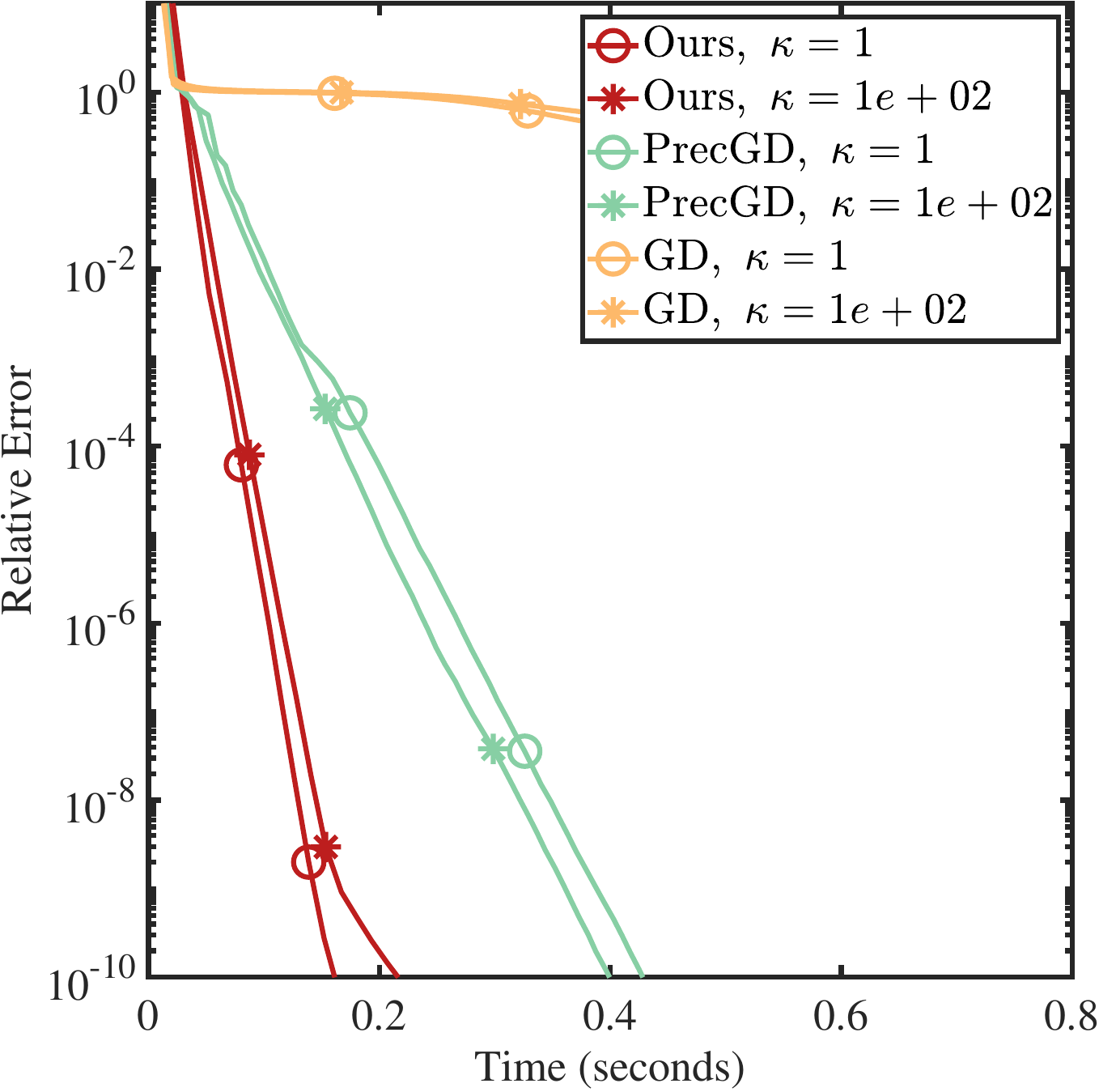}
\end{minipage}%
}
\centering
\caption{Experiments on the weighted PCA task with the following parameter settings: $n_1 = n_2 = 1000$, true rank $r_\star = 5$, estimated rank $r = 2r_\star$. The step size for APGD is set to $\eta = 0.9$, while for the other two methods it is set to $\eta = 0.5$. Subfigure (a) compares the recovery error of the three methods under varying condition numbers. Subfigure (b) presents the comparison of computation time.}
\label{fig:5}
\end{figure}

\subsubsection{1-bit matrix completion}
The 1-bit matrix completion problem is defined as recovering a rank-$r_\star$ matrix $X_\star$ from the 1-bit observation $X_{ij}$ where $X_{ij}=1$ with  probability $\sigma(X_\star)$  and $X_{ij}=0$ with probability $1-\sigma(X_\star)$ and $\sigma(\cdot)$ denotes the sigmoid function. 
After a number of measurements have been taken, define $\alpha_{ij}$ as the fraction of observations in which the $(i,j)$-th entry equals 1. Then we can recover $X_\star$ by minimizing the following objective function:
\begin{equation}
\underset{L \in \mathbb{R}^{n_1 \times r},\ R \in \mathbb{R}^{n_2 \times r}}{\operatorname{minimize}} \sum_{i=1}^{n_1} \sum_{j=1}^{n_2} \left(\log (1+ e^{(LR^\top)_{ij}}) -\alpha_{ij} {(LR^\top)_{ij}} \right).
\label{equ:15}
\end{equation}

Following the setup in \cite{zhang2023preconditioned}, we assume that the number of observations $m$ is large enough so that $\alpha_{ij} = \sigma({(X_\star)}_{ij})$. Under this condition, the optimal solution to (\ref{equ:15}) is exactly $X_\star$. As shown in Figure \ref{fig:6}, APGD achieves faster convergence and lower computation time compared to the other two methods, across different condition numbers.

\begin{figure}[h]
\centering
\subfigure[]{
\begin{minipage}[t]{0.48\linewidth}
\centering
\includegraphics[width=4.3cm,height=4.3cm]{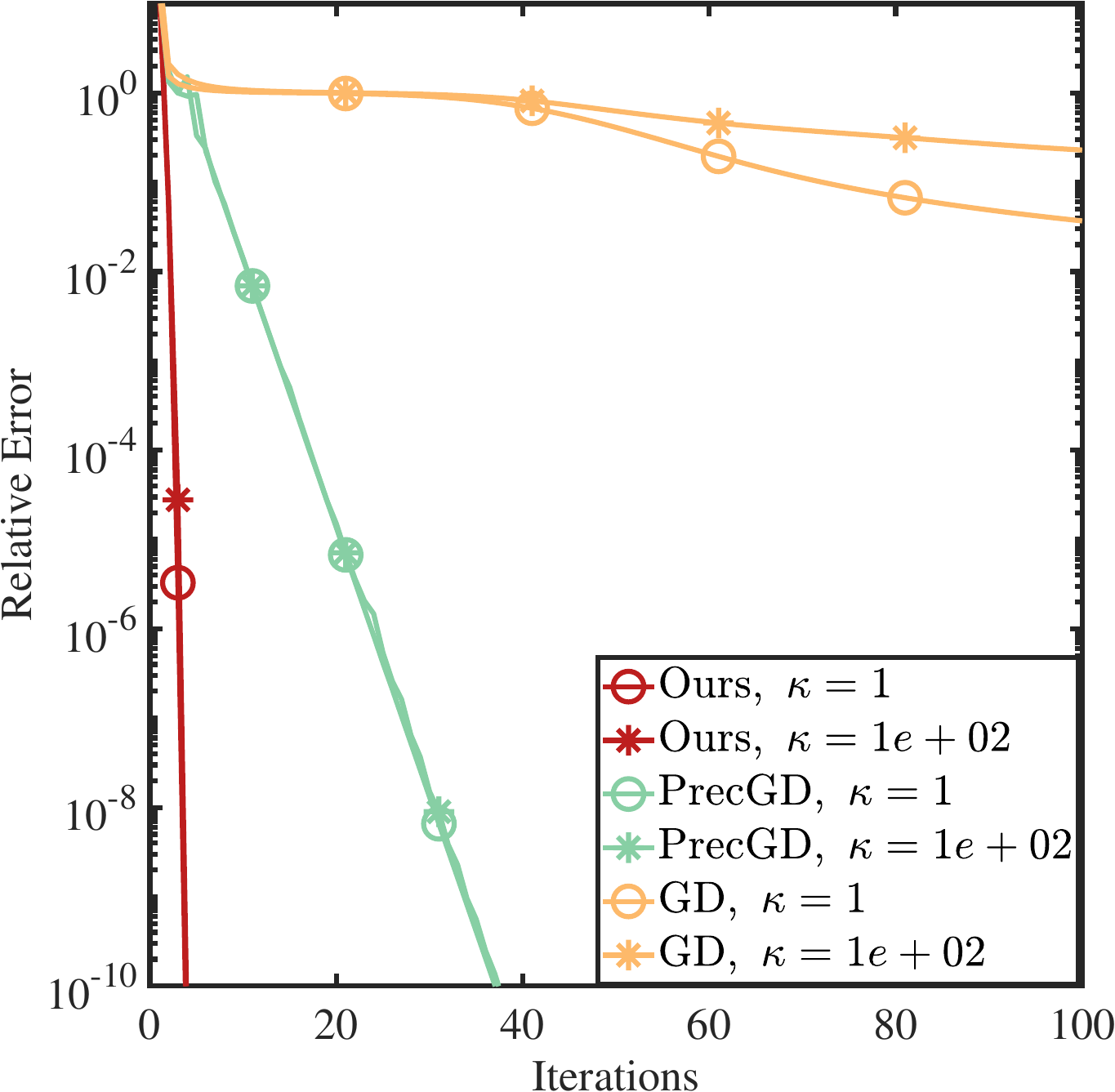}
\end{minipage}%
}
\subfigure[]{
\begin{minipage}[t]{0.48\linewidth}
\centering
\includegraphics[width=4.3cm,height=4.3cm]{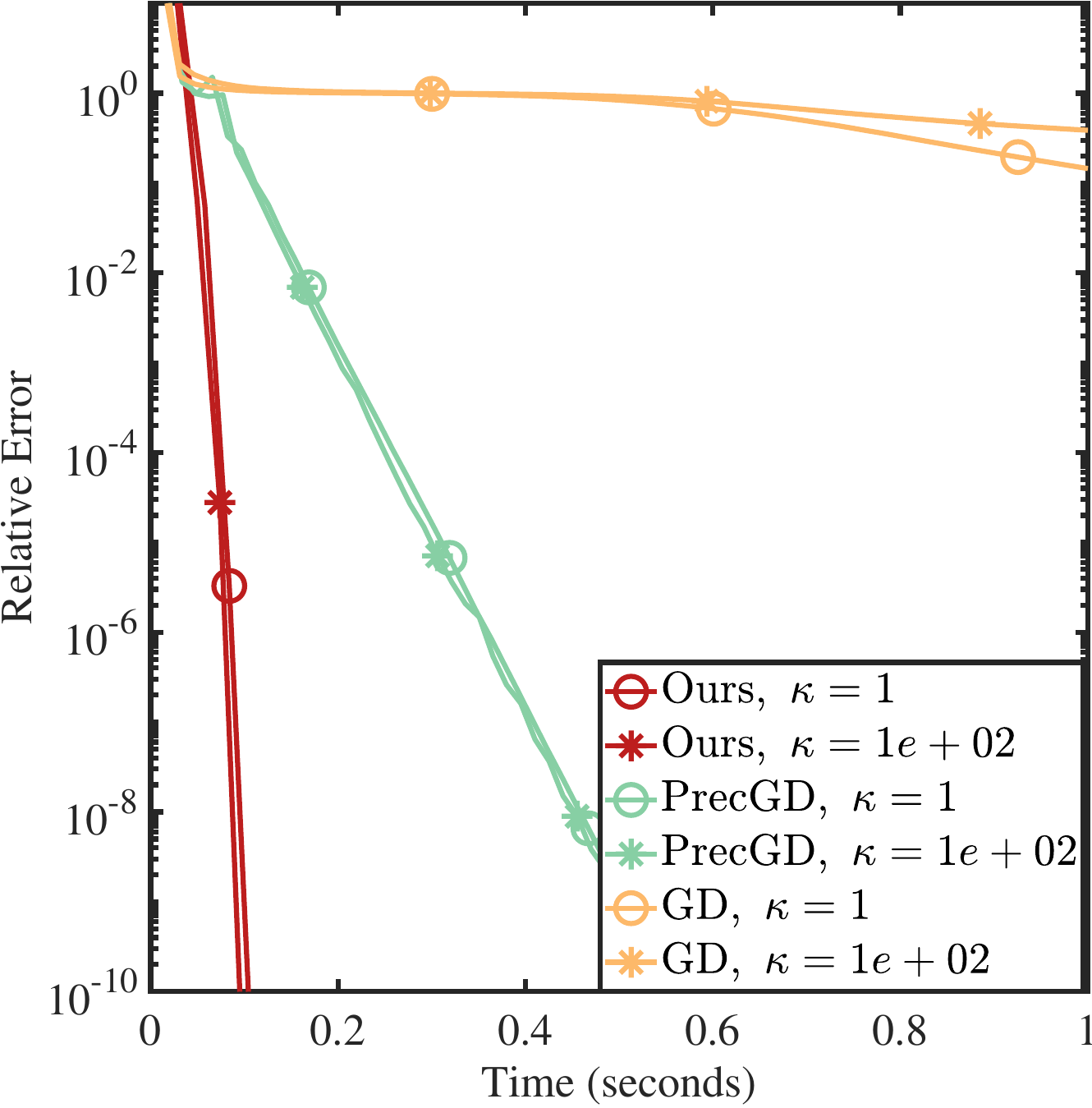}
\end{minipage}%
}
\centering
\caption{Experiments on the 1-bit matrix completion task with the following parameter settings: $n_1 = n_2 = 1000$, true rank $r_\star = 5$, estimated rank $r = 2r_\star$.  The step sizes for each method are tuned to achieve the fastest convergence: APGD uses a step size of 4, while GD uses a step size of 0.5 and NoisyPrecGD uses a step size of 3. Subfigure (a) compares the recovery error of the three methods under varying condition numbers. Subfigure (b) presents the comparison of computation time.
}
\label{fig:6}
\end{figure}

\subsubsection{low-rank matrix completion}

In this section, we conduct real data experiments to verify the effectiveness of APGD. Specifically, similar to the work of Zhang et al. \cite{zhang2024fast}, we perform noisy matrix completion experiments on multispectral images. The noisy matrix completion problem is defined as recovering the ground-truth matrix \( X_\star \) from partial noisy observations \( \mathcal{P}_\Omega(X_\star+ S) \), where 
$$
\mathcal{P}_{\Omega}\left( X \right)_{ij} = \left\{ 
\begin{array}{ll}
    X_{ij}, & \text{if } (i,j) \in \Omega \\ 
    0, & \text{otherwise}
\end{array} 
\right.
$$
and $S$ denotes the Gaussian noise, and \( \Omega \) is generated according to a Bernoulli model, meaning that each entry \((i, j) \in \Omega\) is independently selected with probability \( p \).

Based on the Burer–Monteiro factorization, our optimization problem is formulated as 
\begin{equation}
    \underset{L \in\mathbb{R}^{n_1\times r},\ R \in\mathbb{R}^{n_2\times r}}{\arg \min } \frac{1}{2p} \|\mathcal{P}_\Omega(LR^\top - M)\|_F^2,
\end{equation}
where $M=X_\star + S$.
We can also apply APGD to solve this problem. Here, we use a single spectral band of a multispectral image from the CAVE dataset \cite{CAVE_0293}, with a size of \( 512 \times 512 \). First, we approximate the image with a low-rank matrix of rank 50. For NoisyPrecGD, spectral initialization is applied, while for GD and ScaledGD$(\lambda)$, small random initializations are used as required in the original text. Although spectral initialization is theoretically required for APGD, in practice it is not necessary. Therefore, we adopt random initialization to better highlight the effectiveness of APGD.   All methods are run for only 5 iterations. We use the Signal-to-Noise Ratio (SNR) to measure the level of the noise \( S \), and then evaluate the recovery performance using the Peak Signal-to-Noise Ratio (PSNR), which is displayed below each image.

\textbf{Experiments with different rank $r$}
We begin by evaluating the recovery performance of APGD when transitioning from the exact rank case to the over-parameterized rank case. From Figure \ref{fig:real_verify_rank}, we can observe that APGD successfully recovers the true image in both the exact rank and over-parameterized rank scenarios, even when starting from a random initialization. In contrast, other methods, such as GD and ScaledGD(\(\lambda\)), fail to recover the image. Although NoisyPrecGD also manages to recover the true image, its performance is inferior to that of APGD, and it requires spectral initialization to achieve reasonable results.

\textbf{Experiments with different sampling rate $p$}
We compared the performance of various methods under different sampling rates. As shown in Figure \ref{fig: real_verify_missing_rate}, APGD is capable of approximately recovering the original image even at a low sampling rate ($p=0.2$), whereas other methods failed. As the sampling rate increases, the recovery quality improves significantly.

\begin{figure}[h]
\vskip 0.2in
\begin{center}
\centerline{\includegraphics[width=9cm,height=8cm]{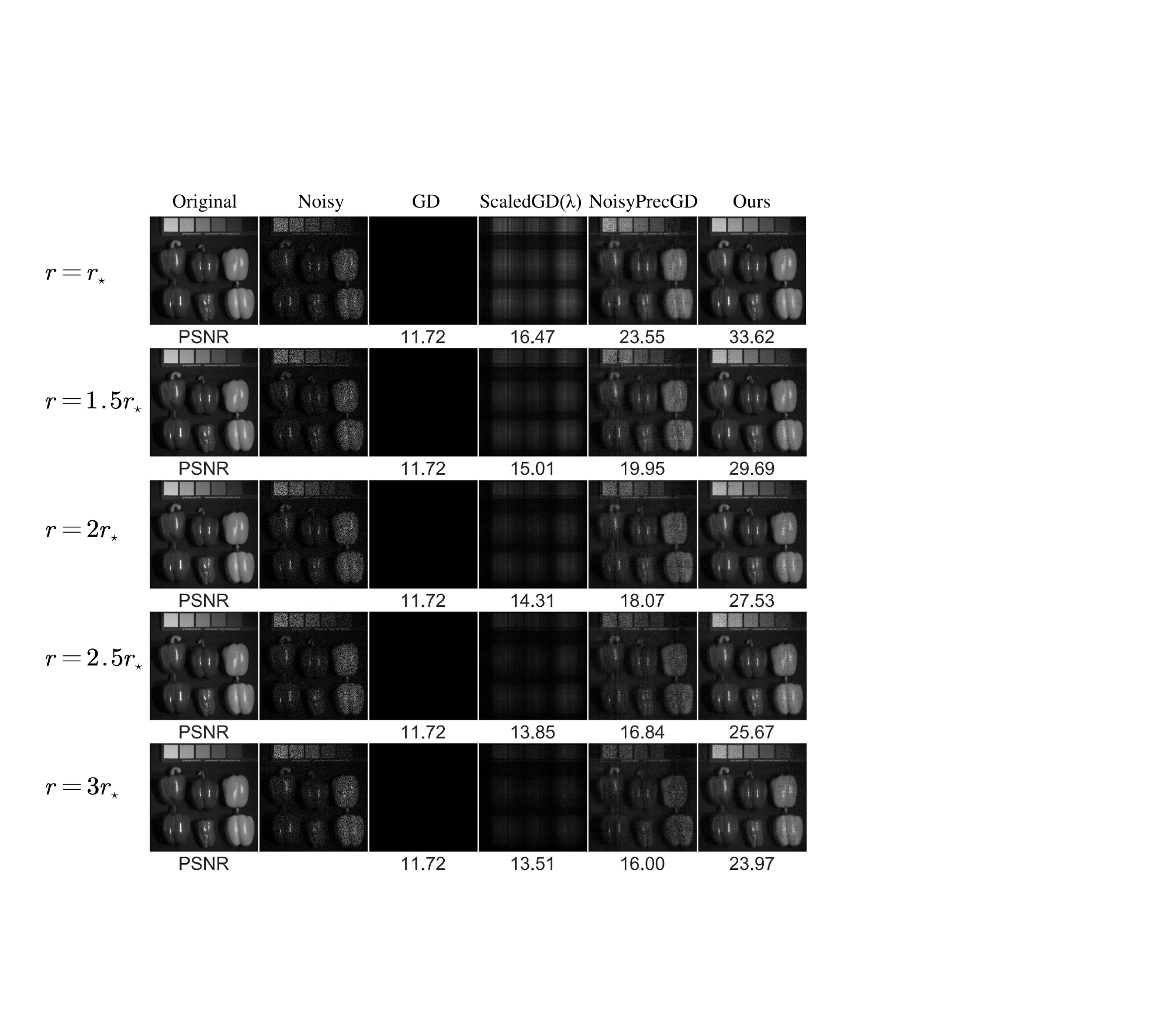}}
\caption{Compare the recovery performance of different algorithms under various over-parameterized ranks $r$, where the SNR of noise is 30.}
\label{fig:real_verify_rank}
\end{center}
\vskip -0.2in
\end{figure}

\begin{figure}[h]
\vskip 0.2in
\begin{center}
\centerline{\includegraphics[width=9cm,height=8cm]{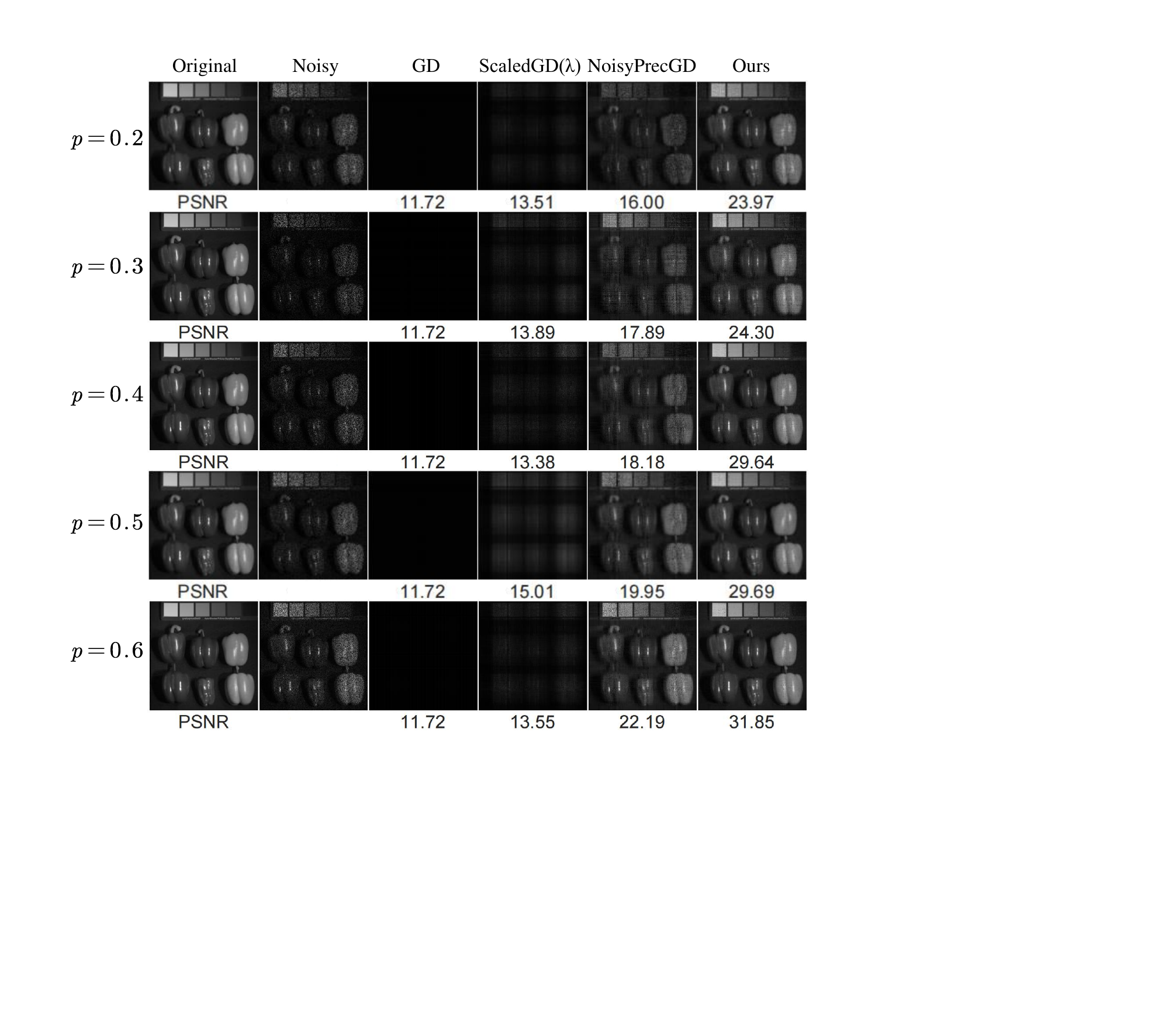}}
\caption{Compare the recovery performance of different algorithms under various sampling rate $p$, where the SNR of noise is 30.}
\label{fig: real_verify_missing_rate}
\end{center}
\vskip -0.2in
\end{figure}

\section{Conclusion}
To deal with the noisy matrix sensing problem, we introduce the APGD algorithm, which could accelerate convergence rate compared to vanilla gradient descent, particularly in the scenarios with large condition numbers and over-parameterization. Both theoretical analysis and empirical studies are conducted to show that APGD achieves near-optimal recovery error at a linear rate. A major strength of APGD is that it removes the need for the damping term used in earlier preconditioning techniques, thus simplifying implementation by avoiding complex parameter tuning. In addition, APGD is stable across a wide range of step sizes and supports larger steps, making it substantially faster than the existing alternatives. Beyond noisy matrix sensing, we demonstrate that APGD is also applicable to a variety of low-rank matrix estimation problems. Precisely, When the loss function satisfies certain geometric conditions, APGD maintains the same linear convergence behavior as that for noisy matrix sensing. A series of experiments are conducted on both synthetic and real-world datasets, including weighted PCA, 1-bit matrix completion, and matrix completion, further validate the efficiency and flexibility of APGD.

\bibliographystyle{IEEEtran}
\bibliography{reference}
\appendix

\subsection{Preliminaries}
We begin by presenting a lemma that bridges the assumptions of Theorem 2 with those of Lemma 2 and Lemma 3.
\begin{lemma}
    Suppose that  we have $m\ge C_\delta \frac{v^2 (2r+1)n \log n }{ \sigma_{r_\star}(X_\star) \rho^2 \delta_{2r+1}^2}$ with constant $\delta_{2r+1}\le \frac{\rho}{8\kappa\sqrt{r_\star + r}},\ \rho\le \frac{1}{2}$. Then with probability at least $1-3 n^{-c_1} -2e^{-c_2 m\delta_{2r+1}} $, the following states holds:
    (1) the linear map $\A(\cdot)$ satisfies rank-$(2r+1) $ RIP with constant $\delta_{2r+1}$;
    (2) the noise terms \begin{equation}
    \|\A^*(s)R_t\|_{P_{R_t}^*}^2 \le \mathcal{E}_{opt},\ \  \|\A^*(s)L_t^\top\|_{P_{L_t}^*}^2 \le \mathcal{E}_{opt}, 
\notag \end{equation} where $\mathcal{E}_{opt}=C_e \frac{\nu^2rn\log n }{m}$ and $n=\max\{n_1,n_2\}$;
    (3) the initial point $X_0$ produced by algorithm \ref{algorithm} satisfies 
    $$
    ||X_0 -X_\star||_F\le \rho \sigma_{r_\star}(X_\star);
    $$

\label{lemma:4}
\end{lemma}
\begin{IEEEproof}
First, according to Theorem 2.3 in \cite{candes2011tight}, if $m \ge \mathcal{O}((2r+1)n / \delta_{2r+1}^2)$, then the operator $\mathcal{A}(\cdot)$ satisfies the rank-$(2r+1)$ RIP with probability at least $1 - e^{-c_2 m \delta_{2r+1}^2}$.

Then for the noise term, with probability at least $1-3n^{-c_1}$, we have
\begin{equation}
\begin{aligned}
\|\A^*(s)R_t\|_{P_{R_t}^*}^2 &= \left \| \left(\sum_{i=1}^m s_i A_i \right) R_t(R_t^\top R_t)^{\dagger/2} \right \|_F^2 \\
&\le \left \| \sum_{i=1}^m s_i A_i \right \|_2^2 \left \| R_t(R_t^\top R_t)^{\dagger/2} \right\|_F^2 \\
&\overset{(a)}{\le}  r \left\| \sum_{i=1}^m s_i A_i \right\|_2^2 \  \overset{(b)}{\le} C_e \frac{\nu^2rn\log n }{m} = \mathcal{E}_{opt},
\end{aligned}
\notag
\end{equation}
where $(a)$ uses the fact that $\left \| R_t(R_t^\top R_t)^{\dagger/2} \right\|_F^2 = \sum_i^r \frac{\sigma_i^2(R_t)}{\sigma_i^2(R_t)}=r$  ; (b) follows from the Lemma 16 in \cite{zhang2021preconditioned}.
The upper bound of $\|\A^*(s)L_{t}^\top\|_{P_{L_t}^*}^2$ can be obtained using a similar method. Combining the first two terms and following the proof of Proposition 23 in \cite{zhang2021preconditioned}, we can conclude that the third term also holds.
\end{IEEEproof}

Then we present a lemma related to RIP, which will be important for the proofs that follow.
\begin{lemma}
Suppose that the linear map $\A(\cdot)$ satisfies the rank$-(2r+1)$ RIP with constant $\delta_{2r+1}$, then we have
\begin{equation}
|| (\mathcal{I}-\A^*\A)(X) ||_F \le \delta_{2r+1}\sqrt{2r} ||X||_F
\notag
\end{equation}
for any matrix $X$ with rank $2r$.
\label{lemma:RIP}
\end{lemma}

\begin{IEEEproof}

This lemma extends Lemma 7.3 from \cite{stoger2021small}, and its proof follows directly by incorporating the norm inequality $|| X ||_F \le \sqrt{2r}||X ||$ into the original argument.
\end{IEEEproof}

\subsection{Proof of Lemmas \ref{Lemma: noiseless lipschitz} and \ref{lemma:lipschitz}}
\label{Proof of Lipschitz-like inequality}
\begin{IEEEproof}
Based on the update rule of APGD, we have
\begin{equation}
\begin{aligned}
&f_c(L_{t+1},R_t) = \frac{1}{2}\| \mathcal{A}(L_{t+1}R_t^\top-X_\star)\|_2^2 \\
&= \frac{1}{2}\| \mathcal{A}(L_{t}-\eta \nabla_L f(L_{t},R_t)(R_t^\top R_t)^\dagger )R_t^\top-X_\star)\|_2^2 \\
&= \frac{1}{2} \left \langle \mathcal{A}(L_tR_t^\top-X_\star)- \eta \mathcal{A}(\nabla_L f(L_{t},R_t)(R_t^\top R_t)^\dagger R_t^\top) \right \rangle \\
& - \eta \mathcal{A}(\nabla_L f(L_{t},R_t)(R_t^\top R_t)^\dagger R_t^\top) , \mathcal{A}(L_tR_t^\top-X_\star)\\
& =\underbrace{\frac{1}{2}\| \mathcal{A}(L_tR_t^\top-X_\star) \|_2^2}_{f_c(L_t,R_t)} + \underbrace{\frac{\eta ^2}{2} \|\mathcal{A}(\nabla_L f(L_{t},R_t)(R_t^\top R_t)^\dagger R_t^\top)\|_2^2}_{Z_1} \\
&-\underbrace{\eta\langle \mathcal{A}(L_tR_t^\top-X_\star), \mathcal{A}(\nabla_L f(L_{t},R_t)(R_t^\top R_t)^\dagger R_t^\top)\rangle }_{Z_2}. 
\end{aligned}
\notag
\end{equation}
For $Z_1$, we have
\begin{equation}
\begin{aligned}
&Z_1 \overset{(a)}{\le} \frac{\eta^2(1+\delta_{2r+1})}{2}\|\nabla_L f(L_{t},R_t)(R_t^\top R_t)^\dagger  R_t^\top\|_F^2 \\
& \overset{(b)}{\le} \frac{\eta^2(1+\delta_{2r+1})}{2} \|\A^*(\A(L_tR_t^\top)-y)R_t(R_t^\top R_t)^{\dagger/2} \|_F^2  \\ 
& = \frac{\eta^2(1+\delta_{2r+1})}{2} \|\A^*(\A(L_tR_t^\top-X_\star)-s) R_t(R_t^\top R_t)^{\dagger/2} \|_F^2 \\
& \le \frac{\eta^2(1+\delta_{2r+1})}{2} \|\A^*(\A(L_tR_t^\top-X_\star)) R_t(R_t^\top R_t)^{\dagger/2} \|_F^2 \\
& \ \ \ \ + \frac{\eta^2(1+\delta_{2r+1})}{2} \|\A^*(s) R_t(R_t^\top R_t)^{\dagger/2} \|_F^2 \\
& = \frac{\eta^2(1+\delta_{2r+1})}{2}  \| \underbrace{\A^*(\A(L_tR_t^\top-X_\star)) R_t}_{\nabla_L f_c(L_t,R_t)} \|_{P_{R_t}^*}^2 \\
&\ \ \ \ + \frac{\eta^2(1+\delta_{2r+1})}{2}\| \A^*(s) R_t \|_{P_{R_t}^*}^2  ,
\end{aligned}
\notag
\end{equation}
where $(a)$ follows the assumption that $\A(\cdot)$ satisfies the rank-$(2r+1)$ RIP; $(b)$ uses the fact that $\|AB\|_F\le \|A\|_F\|B\|_2$.

For $Z_2$, we have
\begin{equation}
\begin{aligned}
Z_2 &= \eta \langle \nabla_Lf(L_{t+1},R_t)(R_t^\top R_t)^\dagger  R_t^\top ,  \A^*\mathcal{A}(L_tR_t^\top-X_\star)\rangle\\
& = \eta\langle  \A^*\mathcal{A}(L_tR_t^\top-X_\star) R_t(R_t^\top R_t)^\dagger  R_t^\top, \A^*\A (L_tR_t^\top-X_\star )\rangle \\
& - \langle \A^*(s)R_t(R_t^\top R_t)^\dagger  R_t^\top, \A^*\A (L_tR_t^\top-X_\star )\rangle \\
& = \eta\|\nabla_L f_c(L_t,R_t)\|^2_{P_{R_t}^* } \\
& - \eta\langle  \A^*(s)R_t(R_t^\top R_t)^\dagger  R_t^\top, \A^*\A (L_tR_t^\top-X_\star )\rangle \\
&=\eta\|\nabla_L f_c(L_t,R_t)\|^2_{P_{R_t}^* } \\
& - \eta\langle  \A^*(s)R_t(R_t^\top R_t)^{\dagger/2}  , \A^*\A (L_tR_t^\top-X_\star )R_t(R_t^\top R_t)^{\dagger/2} \rangle \\
&\ge \eta\|\nabla_L f_c(L_t,R_t)\|^2_{P_{R_t}^* } \\
& - \eta\|\nabla_L f_c(L_t,R_t)\|_{P_{R_t}^* } \|\A^*(s)R_t\|_{P_{R_t}^* }.
\end{aligned}
\notag
\end{equation}
Combining the bounds for $Z_1$ and $Z_2$, we get
\begin{equation}
\begin{aligned}
& f_c(L_{t+1},R_t)  \le f_c(L_t,R_t) \\
& + \frac{\eta^2(1+\delta_{2r+1})}{2} \left( \| \nabla_L f_c(L_t,R_t) \|_{P_{R_t}^*}^2 + \| \A^*(s) R_t \|_{P_{R_t}^*}^2  \right) \\
& -\eta\|\nabla_L f_c(L_t,R_t)\|^2_{P_{R_t}^* } + \eta\|\nabla_L f_c(L_t,R_t)\|_{P_{R_t}^* } \|\A^*(s)R_t\|_{P_{R_t}^* }\\
&\overset{(a)}{\le} f_c(L_t,R_t) \\
&\ \ \ \ \ \  - \underbrace{\left( \eta - \frac{\eta}{3}(1 + 2\eta (1+\delta_{2r+1})  ) \right)}_{C_2} \|\nabla_L f_c(L_t,R_t)\|^2_{P_{R_t}^* },
\end{aligned}
\notag
\end{equation}
where $(a)$ uses the assumption that $\| \nabla_L f_c(L_t,R_t) \|_{P_{R_t}^*}\ge 3 \| \mathcal{A}^*(s)R_t \|_{P_{R_t}^*}$.

Similarly, for $f_c(L_{t+1},R_{t+1})$, we can also deduce that
\begin{equation}
\begin{aligned}
&f_c(L_{t+1},R_{t+1})  \le f_c(L_{t+1},R_t) \\
&\ \ \ \ \ - \left(\eta - \frac{\eta}{3}(1 + 2\eta(1+\delta_{2r+1})  )\right)  \| \nabla_R f_c(L_{t+1},R_t) \|^2_{P_{L_{t+1}}^*}.
\end{aligned}
\notag
\end{equation}
Therefore, we complete the proof of Lemma \ref{lemma:lipschitz}. 
As for Lemma \ref{Lemma: noiseless lipschitz}, by setting the noise $\textbf{s}
 = 0$, we can directly derive Lemma \ref{Lemma: noiseless lipschitz} as a special case of Lemma \ref{lemma:lipschitz}.

\end{IEEEproof}
%%%%%%%%%%%%%%%%%%%%%%%%%%%%%%%%%%%%%%%%%%%%%%%%%%%%%%%%%%%%%%%%%%%%%%%%%%%%%%%%%%%%%%%%
%%%%%%%%%%%%%%%%%%%%%%%%%%%%%%%%%%%%%%%%%%%%%%%%%%%%%%%%%%%%%%%%%%%%%%%%%%%%%%%%%%%%%%%%

\subsection{Proof of Lemma \ref{lemma: gradient dominance}}
\label{Proof of the gradient dominance}
\begin{IEEEproof}
Before proving this lemma, we first define the angle between the column space of $(L_t R_t - X_\star)^\top$ and the column space of $R_t$:
$$
\cos \theta_R^t = \frac{|| (L_t R_t - X_\star) R_t(R_t^\top R_t)^{\dagger/2} ||_F}{||L_t R_t - X_\star||_F}
.$$
Similarly, we define the angle between the column space of $(L_{t+1}R_t-X_\star)$ and the column space of $L_{t+1}$ as 
$$
\cos \theta_L^{t+1} = \frac{|| (L_{t+1} R_t - X_\star)^\top L_{t+1}(L_{t+1}^\top L_{t+1})^{\dagger/2} ||_F}{||L_{t+1} R_t - X_\star||_F}.
$$

Then, we relate $\cos \theta_R^t$ to $\|\nabla_L f_c (L_t,R_t)\|_{P_{R_t}^*}$.
\begin{equation}
\begin{aligned}
&\|\nabla_L f_c (L_t,R_t)\|_{P_{R_t}^*} =\|\A^*\A(L_tR_t^\top-X_\star)R_t(R_t^{\top}R_t)^{\dagger/2}\|_F \\
& \overset{(a)}{\ge} ||(L_tR_t^\top-X_\star)R_t(R_t^{\top}R_t)^{\dagger/2}||_F \\
&\ \ \ \ -||(\mathcal{I}- \A^*\A)(L_tR_t^\top-X_\star)R_t(R_t^{\top}R_t)^{\dagger/2} ||_F \\
&\overset{(b)}{\ge} ||(L_tR_t^\top-X_\star)R_t(R_t^{\top}R_t)^{\dagger/2}||_F\\
&\ \ \ \ -||(\mathcal{I}- \A^*\A)(L_tR_t^\top-X_\star)||_F ||R_t(R_t^{\top}R_t)^{\dagger/2}|| \\
&\overset{(c)}{\ge} ||(L_tR_t^\top-X_\star)R_t(R_t^{\top}R_t)^{\dagger/2}||_F \\
&\ \ \ \ - \sqrt{r+r_\star} \delta_{2r+1} || (L_tR_t^\top-X_\star) ||_F \\
&\overset{(d)}{=} (\cos\theta_R^t - \sqrt{r+r_\star} \delta_{2r+1} )  || (L_tR_t^\top-X_\star) ||_F,
\notag
\end{aligned}
\notag \end{equation}
where $(a)$ uses the norm triangle inequality; ($b$) uses the fact that $||AB||_F\le ||A||_F||B||$; $(c)$ uses the result form Lemma \ref{lemma:RIP} that 
$$
||(\mathcal{I}- \A^*\A)(L_tR_t^\top-X_\star)||_F \le \sqrt{r+r_\star} \delta_{2r+1} || L_tR_t^\top-X_\star ||_F;
$$(d) uses the definition of $\cos\theta_R^t$.

Using a similar argument, we can obtain 
\begin{equation}
\begin{aligned}
& \|\nabla_R f_c (L_{t+1},R_t)\|_{P_{L_{t+1}}^*} \\
&\ge (\cos\theta_L^{t+1} -\sqrt{r+r_\star} \delta_{2r+1} ) || (L_{t+1}R_t^\top-X_\star) ||_F.
\end{aligned}
\notag
\end{equation}

Then, we need to establish a lower bound for $\cos\theta_R^t$ and $\cos\theta_L^{t+1}$. However, directly bounding $\cos\theta_R^t$ and $\cos\theta_L^{t+1}$ from below is rather complicated. Our strategy is to first find an upper bound for \(\sin \theta\), and then use the identity \(\cos^2 \theta + \sin^2 \theta = 1\) to derive a lower bound for \(\cos \theta\).

According to the definitions of  $\cos\theta_R^t$ and $\cos\theta_L^{t+1}$, we have
\begin{equation}
\begin{aligned}
\sin\theta_R^t &= \frac{|| (L_t R_t^\top - X_\star) [I - R_t(R_t^\top R_t)^{\dagger}R_t^\top] ||_F}{||L_t R_t^\top - X_\star||_F},\\ 
\sin\theta_L^{t+1} &= \frac{|| (L_{t+1} R_t^\top - X_\star)^\top[I- L_{t+1}(L_{t+1}^\top L_{t+1})^{\dagger} L_{t+1}^\top ]||_F}{||L_{t+1} R_t^\top - X_\star||_F}.
\end{aligned}
\notag
\end{equation}

Below, we provide upper bounds for $\sin\theta_R^t$ and $\sin\theta_L^{t+1}$, based on the initialization conditions.

\begin{lemma}
 Suppose that we have $|| L_0R_0^\top-L_\star R_\star^\top ||_F \le \rho \sigma_{r_\star}(L_\star R_\star^\top)$ with $\rho \le \frac{1}{2}$, then we have
$$
\sin\theta_R^t \le  \frac{\sqrt{2}\rho}{\sqrt{1-\rho^2}},\ \ \  \sin\theta_L^t \le  \frac{\sqrt{2}\rho}{\sqrt{1-\rho^2}}. 
$$

\label{lemma:theta}
\end{lemma}
\begin{IEEEproof}

Define  matrix $F = \left[\begin{matrix}
    L\\
    R
\end{matrix}\right]\in\mathbb{R}^{(n_1+n_2)\times r}, \ L\in\mathbb{R}^{n_1\times r},\ R\in\mathbb{R}^{n_2\times r}$, $X=LR^\top$, $F_\star=\left[\begin{matrix}
    L_\star\\
    R_\star
\end{matrix}\right]\in\mathbb{R}^{(n_1+n_2)\times r_\star},\ X_\star = U_\star\Sigma_\star V^\top_\star, \ L_\star = U_\star\Sigma_\star^{1/2}\in\mathbb{R}^{n_1\times r_\star},\ R_\star=V_\star\Sigma_\star^{1/2}\in\mathbb{R}^{n_2\times r_\star}.$

The proof of this lemma is based on the result of Lemma 13 from \cite{zhang2021preconditioned}. We first present Lemma 13 from \cite{zhang2021preconditioned}.
\begin{lemma}[Lemma 13 in \cite{zhang2021preconditioned}]
Suppose that $|| FF^\top - F_\star F_\star^\top ||_F \le \rho \sigma_{r_\star} (F_\star ^\top F_\star)$ with $\rho\le 1/\sqrt{2}$, then we have
$$
\frac{|| X_\star ||_F}{|| FF^\top - F_\star F_\star^\top ||_F} \le \frac{\rho}{ \sqrt{2} \sqrt{1-\rho^2}}.
$$ 
\label{lemma:13}
\end{lemma}

First, we prove that the initialization condition in Lemma \ref{lemma:13} is satisfied.
For $FF^\top-F_\star F_\star^\top$, we have
\begin{equation}
\begin{aligned}
|| FF^\top - F_\star F_\star^\top ||_F & \overset{(a)}{\le}  2||LR^\top-L_\star R_\star^\top ||_F \\
&\overset{(b)}{\le } 2\rho\sigma_{r_\star}(X_\star) \overset{(c)}{= } \sigma_{r_\star}(F_\star^\top F_\star),
\end{aligned}
\notag
\end{equation}
where $(a)$ follows from the result of Lemma 24 in \cite{tong2021accelerating}; $(b)$ uses the initialization assumption $||LR^\top-L_\star R_\star^\top ||_F \le  \rho\sigma_{r_\star}(X_\star)$; $(c)$ uses the fact that 
$$
\sigma_{r_\star}(F_\star ^\top F_\star) = \sigma_{r_\star}(L_\star^\top L_\star +R_\star^\top R_\star) = 2\sigma_{r_\star}(\Sigma_\star)=2\sigma_{r_\star}(X_\star).
$$

Next, we use the result of Lemma \ref{lemma:13} to prove Lemma \ref{lemma:theta}.

For $|| (L_t R_t^\top - X_\star) [I - R_t(R_t^\top R_t)^{\dagger}R_t^\top] ||_F$ in $\sin\theta_R^t$, we have
\begin{equation}
\begin{aligned}
&|| (L_t R_t^\top - X_\star) [I - R_t(R_t^\top R_t)^{\dagger}R_t^\top] ||_F \\
& = || X_\star [I - R_t(R_t^\top R_t)^{\dagger}R_t^\top] ||_F \\
&\le ||X_\star||_F ||I - R_t(R_t^\top R_t)^{\dagger}R_t^\top|| \le ||X_\star||_F.
\end{aligned}
\label{equ:25}
\end{equation}
For $||L_t R_t^\top - X_\star||_F$ in $\sin\theta_R^t$, we have
\begin{equation}
||LR^\top-L_\star R_\star^\top ||_F \ge  \frac{1}{2}|| FF^\top - F_\star F_\star^\top ||_F,
\label{equ:26}
\end{equation}
where this inequality follows from the result of Lemma 24 in \cite{tong2021accelerating}.
Therefore, combining equations (\ref{equ:25}) and (\ref{equ:26}), we have
\begin{equation}
\begin{aligned}
&\sin\theta_R^t = \frac{|| (L_t R_t^\top - X_\star) [I - R_t(R_t^\top R_t)^{\dagger}R_t^\top] ||_F}{||L_t R_t^\top - X_\star||_F} \\
&\le \frac{2||X_\star||_F}{||F_tF_t^\top - F_\star F_\star^\top||_F}\le \frac{\sqrt{2} \rho}{  \sqrt{1-\rho^2}}.
\end{aligned}
\notag
\end{equation}
Similarly, for $\sin\theta_L^{t}$, we have
\begin{equation}
\sin\theta_L^t \le  \frac{\sqrt{2}\rho}{\sqrt{1-\rho^2}}.
\notag \end{equation}
Thereby, we complete the proof of Lemma \ref{lemma:theta}.

\end{IEEEproof}

Based on the upper bounds of $\sin\theta_R^t$ and $\sin\theta_L^{t+1}$, we have
$$
\cos\theta_R^t \ge \sqrt{\frac{1-3\rho^2}{1-\rho^2}},\ \ \cos\theta_L^{t+1} \ge \sqrt{\frac{1-3\rho^2}{1-\rho^2}}. 
$$
Therefore, we have
\begin{equation}
\begin{aligned}
&\|\nabla_L f_c (L_t,R_t)\|_{P_{R_t}^*}^2 \\
&\ \ \ \ \ge  (\cos\theta_R^t - \sqrt{r+r_\star} \delta_{2r+1} )^2  || (L_tR_t^\top-X_\star) ||_F^2 \\
&\ \ \ \ \ge \left( \sqrt{\frac{1-3\rho^2}{1-\rho^2}} -\sqrt{r+r_\star} \delta_{2r+1} \right)^2 || (L_tR_t^\top-X_\star) ||_F^2 \\
&\ \ \ \ \overset{(a)}{\ge} \underbrace{\left( \sqrt{\frac{1-3\rho^2}{1-\rho^2}} -\sqrt{r+r_\star} \delta_{2r+1} \right)^2}_{\tau} f_c(L_t,R_t),\\
&\|\nabla_R f_c (L_{t+1},R_t)\|_{P_{L_{t+1}}^*}^2 \\
& \ \ \ \ \ge \underbrace{\left( \sqrt{\frac{1-3\rho^2}{1-\rho^2}} -\sqrt{r+r_\star} \delta_{2r+1} \right)^2}_{\tau} f_c(L_{t+1},R_t),
\end{aligned}
\notag \end{equation}
where $(a)$ uses the $\delta_{2r+1}$-RIP condition
$$
||L_tR_t^\top-X_\star||_F^2\ge  \frac{1}{1+\delta_{2r+1}} ||\A(L_tR_t^\top-X_\star)||_2^2 \ge f_c(L_{t+1},R_t).
$$
Thereby, we complete the proof of Lemma \ref{lemma: gradient dominance}.
\end{IEEEproof}

%%%%%%%%%%%%%%%%%%%%%%%%%%%%%%%%%%%%%%%%%%%%%%%%%%%%%%%%%%%%%%%%%%%%%%%%%%%%%%%%%%%%%%%%%%%%%
%%%%%%%%%%%%%%%%%%%%%%%%%%%%%%%%%%%%%%%%%%%%%%%%%%%%%%%%%%%%%%%%%%%%%%%%%%%%%%%%%%%%%%%%%%%%%

\subsection{Proof of Theorem \ref{main theorem}}
\label{proof of the main results}    
\begin{IEEEproof}
Assuming that the assumptions in Theorem \ref{main theorem} hold, we can conclude that the assumptions in Lemmas \ref{lemma:lipschitz} and \ref{lemma: gradient dominance} also hold by the result of Lemma \ref{lemma:4}.

We then classify $ \| \nabla_L f_c(L_t,R_t) \|_{P_{R_t}^*},\ \|\nabla_R f_c(L_{t+1},R_t)\|_{P_{L_{t+1}}^*} $ into four cases as follows:
\begin{itemize}
    \item (a): $ \| \nabla_L f_c(L_t,R_t) \|_{P_{R_t}^*} >  3\| \mathcal{A}^*(s) R_t \|_{P_{R_t}^*}$, and $\|\nabla_R f_c(L_{t+1},R_t)\|_{P_{L_{t+1}}^*} > 3\|\mathcal{A}^*(s)L_{t+1}^\top \|_{P_{L_{t+1}}^*}$
    \item (b): $ \| \nabla_L f_c(L_t,R_t) \|_{P_{R_t}^*} >  3\| \mathcal{A}^*(s) R_t \|_{P_{R_t}^*}$, and $\|\nabla_R f_c(L_{t+1},R_t)\|_{P_{L_{t+1}}^*} \le 3\|\mathcal{A}^*(s)L_{t+1}^\top \|_{P_{L_{t+1}}^*}$
    \item (c): $ \| \nabla_L f_c(L_t,R_t) \|_{P_{R_t}^*} \le  3\| \mathcal{A}^*(s) R_t \|_{P_{R_t}^*}$, and $\|\nabla_R f_c(L_{t+1},R_t)\|_{P_{L_{t+1}}^*} > 3\|\mathcal{A}^*(s)L_{t+1}^\top \|_{P_{L_{t+1}}^*}$
    \item (d): $ \| \nabla_L f_c(L_t,R_t) \|_{P_{R_t}^*} \le  3\| \mathcal{A}^*(s) R_t \|_{P_{R_t}^*}$, and $\|\nabla_R f_c(L_{t+1},R_t)\|_{P_{L_{t+1}}^*} \le 3\|\mathcal{A}^*(s)L_{t+1}^\top \|_{P_{L_{t+1}}^*}$
\end{itemize}

\textbf{Analysis of case (a)}
For case (a), we directly apply the results from Lemma \ref{lemma: gradient dominance}, and obtain 
\begin{equation}
\begin{aligned}
& f_c (L_{t+1},R_{t+1}) \le (1-\eta_c)^2 f_c(L_t,R_t), \\
&\|L_{t+1}R_{t+1}^\top - X_\star\|_F^2 \le \frac{1+\delta_{2r+1}}{1-\delta_{2r+1}} (1-\eta_c)^2 \| 
L_tR_t^\top - X_\star \|_F^2,\\  
\end{aligned}
\notag
\end{equation}
where $\eta_c = \tau\left ( \eta - \frac{\eta}{3}(1 + 2\eta(1+\delta_{2r+1}) ) \right)$.

\textbf{Analysis of case (b)}
For case (b), we have
\begin{equation}
f_c(L_{t+1},R_t) \overset{(i)}{\le} \frac{1}{\tau} \|\nabla_R f_c(L_{t+1},R_t)\|_{P_{L_{t+1}}^*}  \le \frac{3}{\tau} \|\mathcal{A}^*(s)L_{t+1}^\top \|_{P_{L_{t+1}}^*}
\notag
\end{equation}
where $(i)$ uses the result form Lemma \ref{lemma: gradient dominance}, i.e., $\|\nabla_R f_c (L_{t+1},R_t)\|_{P_{L_{t+1}}^*}^2 \ge \tau f_c (L_{t+1},R_t)$.
Then for $f_c(L_{t+1},R_{t+1})$, we have
\begin{equation}
\begin{aligned}
&f_c(L_{t+1},R_{t+1})  \le f_c(L_{t+1},R_t) -\eta\|\nabla_R f_c(L_{t+1},R_t)\|^2_{P_{L_{t+1}}^* } \\
&  + \frac{\eta^2(1+\delta_{2r+1})}{2} \| \nabla_R f_c(L_{t+1},R_t) \|_{P_{L_{t+1}}^*}^2 \\
& + \frac{\eta^2(1+\delta_{2r+1})}{2} \| \A^*(s) L_{t+1}^\top \|_{P_{L_{t+1}}^*}^2 \\
&  + \eta\|\nabla_R f_c(L_{t+1},R_t)\|_{P_{L_{t+1}}^* } \|\A^*(s)L_{t+1}^\top\|_{P_{L_{t+1}}^* } \\
& \overset{(i)}{\le} f_c(L_{t+1},R_t) + 3\eta  \|\A^*(s) L_{t+1}^\top \|_{P_{L_{t+1}}^*}^2 \\
&+ 2{\eta^2(1+\delta_{2r+1})} \|\A^*(s) L_{t+1}^\top \|_{P_{L_{t+1}}^*}^2 \\
&\le \frac{1}{\tau} \|\A^*(s) L_{t+1}^\top \|_{P_{L_{t+1}}^*}^2  + 3\eta  \|\A^*(s) L_{t+1}^\top \|_{P_{L_{t+1}}^*}^2 \\  
& + 2{\eta^2(1+\delta_{2r+1})} \|\A^*(s) L_{t+1}^\top \|_{P_{L_{t+1}}^*}^2\\
&\overset{(ii)}{<} \left(\frac{1}{\tau} +7\right)\|\A^*(s) L_{t+1}^\top\|_{P_{L_{t+1}}^*}^2
\end{aligned}
\notag
\end{equation}
where $(i)$ uses the assumption that $\|\nabla_R f_c(L_{t+1},R_t)\|_{P_{L_{t+1}}^*} \le 3\|\mathcal{A}^*(s) L_{t+1}^\top\|_{P_{L_{t+1}}^*}$; (ii) uses the fact that $\delta_{2r+1}<1$ and $\eta<1$.

\par

\textbf{Analysis of case (c)}
For case (c), we have 
\begin{equation}
f_c(L_{t},R_t) \overset{(i)}{\le} \frac{1}{\tau} \|\nabla_L f_c(L_{t},R_t)\|_{P_{R_{t}}^*}  \le \frac{3}{\tau} \|\mathcal{A}^*(s)R_t\|_{P_{R_{t}}^*},
\notag
\end{equation}
where (i) uses the result from Lemma \ref{lemma: gradient dominance}, i.e., $\|\nabla_L f_c (L_{t},R_t)\|_{P_{R_{t}}^*}^2 \ge \tau f_c (L_{t},R_t)$.
For $f_c(L_{t+1},R_t)$, we have
\begin{equation}
\begin{aligned}
& f_c(L_{t+1},R_t) \le f_c(L_t,R_t) -\eta\|\nabla_L f_c(L_{t},R_t)\|^2_{P_{R_{t}}^* }\\
&\ \ \ \ +\frac{\eta^2(1+\delta_{2r+1})}{2} \left( \| \nabla_L f_c(L_{t},R_t) \|_{P_{R_{t}}^*}^2 + \| \A^*(s) R_{t} \|_{P_{R_{t}}^*}^2  \right) \\
&\ \ \ \   + \eta\|\nabla_L f_c(L_{t},R_t)\|_{P_{R_{t}}^* } \|\A^*(s)R_{t}\|_{P_{R_{t}}^* } \\
& \overset{(i)}{\le} f_c(L_{t},R_t) +2\eta^2(1+\delta_{2r+1}) \|\A^*(s) R_{t} \|_{P_{R_{t}}^*}^2 \\
&\ \ \ \  + 3\eta  \|\A^*(s) R_{t} \|_{P_{R_{t}}^*}^2 \\
&\le \frac{1}{\tau} \|\A^*(s) R_{t} \|_{P_{R_{t}}^*}^2 + 2\eta^2(1+\delta_{2r+1}) \|\A^*(s) R_{t} \|_{P_{R_{t}}^*}^2 \\
&\ \ \ \ + 3\eta  \|\A^*(s) R_{t} \|_{P_{R_{t}}^*}^2 \\  
&\overset{(ii)}{<} \left(\frac{1}{\tau} +7\right)\|\A^*(s) R_{t} \|_{P_{R_{t}}^*}^2,
\end{aligned}
\label{equ:18}
\end{equation}
where $(i)$ uses the assumption that $\|\nabla_L f_c(L_{t},R_t)\|_{P_{R_{t}}^*} \le 3\|\mathcal{A}^*(s) R_t \|_{P_{R_{t}}^*}$; (ii) uses the fact that $\delta_{2r+1}<1$ and $\eta<1$.

And then we have 
\begin{equation}
f_c(L_{t+1},R_{t+1}) \le (1-\eta_c)f_c(L_{t+1},R_t).
\label{equ:19}
\end{equation}

since $\|\nabla_R f_c(L_{t+1},R_t)\|_{P_{L_{t+1}}^*} > 3\|\mathcal{A}^*(s) L_{t+1}^\top \|_{P_{L_{t+1}}^*}$.

Combining equations (\ref{equ:18}) and (\ref{equ:19}), we have
\begin{equation}
    f_c(L_{t+1},R_{t+1}) < \left(\frac{1}{\tau} +7\right)\|\A^*(s) R_{t} \|_{P_{R_{t}}^*}^2.
\notag \end{equation}

\textbf{Analysis of case (d)} The analysis of case (d) is actually the same as case (b), and then we have
\begin{equation}
  f_c(L_{t+1},R_{t+1}) \le   \left(\frac{1}{\tau} +7\right)\|\A^*(s) L_{t+1}^\top\|_{P_{L_{t+1}}^*}^2.
\notag \end{equation}

Therefore, combining the analysis of the four case, we have
$$
f_c (L_{t+1},R_{t+1}) \le (1-\eta_c)^2 f_c(L_t,R_t)
$$ for any $t$ where $ \| \nabla_L f_c(L_t,R_t) \|_{P_{R_t}^*} >  3\| \mathcal{A}^*(s) \|_{P_{R_t}^*}$, and $\|\nabla_R f_c(L_{t+1},R_t)\|_{P_{L_{t+1}}^*} > 3\|\mathcal{A}^*(s) \|_{P_{L_{t+1}}^*}$. Otherwise, we have
\begin{equation}
\begin{aligned}
&f_c(L_{t+1},R_{t+1}) \\
&\ \ \ \  \le   \left(\frac{1}{\tau} +7\right)\max\{\|\A^*(s) L_{t+1}^\top\|_{P_{L_{t+1}}^*}^2,\ \|\A^*(s) R_{t} \|_{P_{R_{t}}^*}^2 \}\\
&\ \ \ \  \overset{(i)}{\le} C_3 \mathcal{E}_{opt},
\end{aligned}
\notag
\end{equation}
where $(i)$ uses the result of Lemma \ref{lemma:4} and $C_3=\frac{1}{\tau} +7$.
This implies that when the gradient is large, the recovery error converges linearly, whereas when the gradient is small, the recovery error is already close to optimal.

\end{IEEEproof}
%%%%%%%%%%%%%%%%%%%%%%%%%%%%%%%%%%%%%%%%%%%%%%%%%%%%%%%%%%%%%%%%%%%%%%%%%%%%%%%%%%%%%%%%%%%%%
%%%%%%%%%%%%%%%%%%%%%%%%%%%%%%%%%%%%%%%%%%%%%%%%%%%%%%%%%%%%%%%%%%%%%%%%%%%%%%%%%%%%%%%%%%%%%

\subsection{Proof of Theorem \ref{theorem:general}}
\label{proof of the general case}
For general low-rank matrix estimation problems, our analysis follows a similar approach to that used for low-rank matrix recovery. Specifically, if the loss function $g$ satisfies restricted smoothness and restricted strong convexity, then we can establish that 
\begin{equation}
\frac{\mu}{2}||X-X_\star||_F^2 \le g(X)-g(X_\star) \le \frac{L_g}{2}||X-X_\star||_F^2.
\label{equ:37}
\end{equation}

We then construct a Lipschitz-like inequality similar to Lemma \ref{lemma:lipschitz}.
\begin{lemma}
For the general low-rank matrix estimation, suppose that the loss function $g$ satisfies the rank-$2r$ restricted $L$-smooth and restricted $\mu$-strongly convex, then we have
\begin{equation}
\begin{aligned}
& g(L_{t+1}R_t^\top) \le g(L_tR_t^\top) - \eta(1-\frac{L_g\eta}{2}) ||\nabla g(L_tR_t^\top)R_t  ||_{P_{R_t}}^2,\\
& g(L_{t+1}R_{t+1}^\top) \le g(L_{t+1}R_t^\top)\\
&\ \ \ \ \ \ \ \ \ \ \ \ \ \ \ \ \ \ \ \ \ - \eta(1-\frac{L_g\eta}{2}) ||\nabla g(L_{t+1}R_t^\top)^\top L_{t+1} ||_{P_{L_{t+1}}}^2.
\end{aligned}
\notag
\end{equation}
\label{lemma:lipschitz-general}
\end{lemma}
\begin{IEEEproof}
Based on the $L_g$-smooth and the update rule of APGD, we have
\begin{equation}
\begin{aligned}
   & g(L_{t+1}R_t^\top) \le g(L_tR_t^\top)  + \frac{L_g}{2} ||L_{t+1}R_t^\top-L_tR_t^\top||_F^2\\
   &+ \langle \nabla g(L_tR_t^\top), L_{t+1}R_t^\top-L_tR_t^\top\rangle\\
   & =g(L_tR_t^\top) + \frac{L_g\eta^2}{2} ||\nabla g(L_tR_t^\top)R_t (R_t^\top R_t)^\dagger R_t^\top ||_F^2 \\
   &- \eta \langle \nabla g(L_tR_t^\top), \nabla g(L_tR_t^\top)R_t (R_t^\top R_t)^\dagger R_t^\top \rangle\\
   &\le g(L_tR_t^\top) + \frac{L_g\eta^2}{2} ||\nabla g(L_tR_t^\top)R_t ||_{P_{R_t}}^2 \\
   &-\eta \langle \nabla g(L_tR_t^\top)R_t (R_t^\top R_t)^{\dagger/2}, \nabla g(L_tR_t^\top)R_t (R_t^\top R_t)^{\dagger/2} \rangle\\
   &\le g(L_tR_t^\top) - \eta(1-\frac{L_g\eta}{2}) ||\nabla g(L_tR_t^\top)R_t ||_{P_{R_t}}^2.
\end{aligned}
\notag
\end{equation}
 Similarly, we have 
 \begin{equation}
\begin{aligned}
    & g(L_{t+1}R_{t+1}^\top) \le g(L_{t+1}R_t^\top) \\
    &\qquad \qquad \ \ \ \ \ - \eta(1-\frac{L_g\eta}{2}) ||\nabla g(L_{t+1}R_t^\top)^\top L_{t+1}||_{P_{L_{t+1}}}^2.
\end{aligned}
 \notag \end{equation}
\end{IEEEproof}

therefore, we complete the proof of Lemma \ref{lemma:lipschitz-general}.

Next, we derive lower bounds for $||\nabla g(L_tR_t^\top)R_t ||_{P_{R_t}}^2$ and $||\nabla g(L_{t+1}R_t^\top)^\top L_{t+1}  ||_{P_{L_{t+1}}}^2$ separately.

\begin{lemma}
Suppose that the loss function $g$ satisfies the rank-$2r$ restricted $L$-smooth and restricted $\mu$-strongly convex, and the initial point $X_0$ satisfies $||X_0-X_\star||_F \le \rho \sigma_{r_\star},\ \rho \le \frac{1}{2}$, then we have
\begin{equation}
\begin{aligned}
   &||\nabla g(L_tR_t^\top)R_t ||_{P_{R_t}}^2 \ge \zeta [g(L_tR_t^\top)- g(X_\star)]\\
&||\nabla g(L_{t+1}R_t^\top)L_{t+1} ||_{P_{L_{t+1}}}^2 \ge \zeta [g(L_{t+1}R_t^\top) - g(X_\star)],
\end{aligned}
\notag
\end{equation}
where $\zeta = \frac{(C_\rho-1)L +(C_\rho +1)\mu }{\sqrt{2L}}.$
\label{lemma:12}
\end{lemma}
\begin{IEEEproof}

The proof of this lemma begins by applying Lemma 15 from \cite{zhang2023preconditioned}.

\begin{lemma}(Lemma 15 in \cite{zhang2023preconditioned})
Suppose that the loss function $g$ satisfies the rank-$r$ restricted $L_g$-smooth and restricted $\mu$-strongly convex, then we have
\begin{equation}
\begin{aligned}
&\left |  \frac{2}{\mu+L_g} \langle \nabla ^2 g(X)[E] , F \rangle - \langle E, F\rangle \right| \le \frac{L_g-\mu}{L_g+\mu} ||E||_F ||F||_F
\end{aligned}
\notag \end{equation}
for all $\operatorname{rank}(M)\le r $ and $\operatorname{rank}(E+F)\le r$, where $\nabla ^2 g(X)[E]=\lim _{t\to 0} \frac{1}{t} [\nabla g(X+tE)-\nabla g(X)]$.
\label{lemma:14}
\end{lemma}

\begin{equation}
\begin{aligned}
&||\nabla g(L_tR_t^\top)R_t (R_t^\top R_t)^{\dagger/2} ||_F \\
&=\underset{||Y||_F=1}{\max}\langle  \nabla g(L_tR_t^\top)R_t (R_t^\top R_t)^{\dagger/2}, Y \rangle\\
&=\underset{||Y||_F=1}{\max}\langle \nabla g(L_tR_t^\top) , Y(R_t^\top R_t)^{\dagger/2}R_t^\top \rangle \\ 
&= \underset{||Y||_F=1}{\max}\langle \nabla g(L_tR_t^\top) - \nabla g(X_\star), Y(R_t^\top R_t)^{\dagger/2}R_t^\top \rangle\\
&\overset{(a)}{=} \underset{||Y||_F=1}{\max} \int_0^1  \langle \nabla^2 g(X_\star + tE_t)[E_t],Y(R_t^\top R_t)^{\dagger/2}R_t^\top \rangle dt \\
&\overset{(b)}{\ge} \underset{||Y||_F=1}{\max} \frac{L_g+\mu}{2} \left[ \langle E_t , Y(R_t^\top R_t)^{\dagger/2}R_t^\top \rangle - \frac{L_g-\mu}{L_g+\mu} ||E_t||_F \right] \\
&{\ge}  \underset{||Y||_F=1}{\max}  \frac{L_g+\mu}{2} \langle E_t , Y(R_t^\top R_t)^{\dagger/2}R_t^\top \rangle - \frac{L_g-\mu}{2} ||E_t||_F,
\end{aligned}
\label{equ:41}
\end{equation}
where $E_t = L_tR_t^\top - X_\star$, and $(a)$ uses the definition of $\nabla ^2 g(X)[E]$; $(b)$ uses the result of Lemma \ref{lemma:14}. Similarly, we have
\begin{equation}
\begin{aligned}
&||\nabla g(L_{t+1}R_t^\top)^\top L_{t+1} (L_{t+1}^\top L_{t+1})^{\dagger/2} ||_F\\
&\ \ \ \ \ge \underset{||Y||_F=1}{\max}  \frac{L_g+\mu}{2} \langle E_{t+\frac{1}{2}}, Y(L_{t+1}^\top L_{t+1})^{\dagger/2}L_{t+1}^\top \rangle\\
&\ \ \ \  - \frac{L_g-\mu}{2} ||E_{t+\frac{1}{2}}||_F,
\end{aligned}
\label{equ:42}
\end{equation}
where $E_{t+\frac{1}{2}}$ denotes $L_{t+1}R_t^\top-X_\star.$ 

Then we need to bound $\underset{||Y||_F=1}{\max}  \frac{L_g+\mu}{2} \langle E_t , Y(R_t^\top R_t)^{\dagger/2}R_t^\top \rangle$， while $\underset{||Y||_F=1}{\max}  \frac{L_g+\mu}{2} \langle E_{t+\frac{1}{2}}, Y(L_{t+1}^\top L_{t+1})^{\dagger/2}L_{t+1}^\top \rangle$ can be bounded in a similar way. Note that the angle between the column space of $E_t^\top$ and that of $R_t$ is $$\cos\theta_R^t = \frac{||E_t R_t(R_t^\top R_t)^{\dagger/2}||_F}{||E_t||_F}= \underset{||Y||_F=1}{\max} \frac{\langle E_tR_t(R_t^\top R_t)^{\dagger/2} , Y \rangle}{||E_t||_F}.$$ If we can lower bound $\cos\theta$ , then we have the lower bound of $\underset{||Y||_F=1}{\max}  \frac{L_g+\mu}{2} \langle E_t , Y(R_t^\top R_t)^{\dagger/2}R_t^\top \rangle$.

Therefore, using the result from Lemma \ref{lemma:theta}, we obtain
$$
\cos \theta_R^t \ge \sqrt{\frac{1-3\rho^2}{1-\rho^2}},\ \ \cos\theta_L^{t+1} \ge \sqrt{\frac{1-3\rho^2}{1-\rho^2}}.  
$$
Furthermore, we can derive
\begin{equation}
\begin{aligned}
\underset{||Y||_F=1}{\max}  \frac{L_g+\mu}{2} \langle E_t , Y(R_t^\top R_t)^{\dagger/2}R_t^\top \rangle &\ge \frac{(L_g+\mu)C_\rho}{2} ||E_t||_F\\
\underset{||Y||_F=1}{\max}  \frac{L_g+\mu}{2} \langle E_{t+\frac{1}{2}} , Y(R_t^\top R_t)^{\dagger/2}R_t^\top \rangle &\ge \frac{(L_g+\mu)C_\rho}{2} ||E_{t+\frac{1}{2}}||_F,
\end{aligned}
\label{equ:43}
\end{equation}
where $C_\rho = \sqrt{\frac{1-3\rho^2}{1-\rho^2}}.$

Combining the results of equation (\ref{equ:41}), (\ref{equ:42}) and (\ref{equ:43}), we have
 \begin{equation}
    \begin{aligned}
& ||\nabla g(L_tR_t^\top)R_t (R_t^\top R_t)^{\dagger/2} ||_F \\
&\ \ \ \ \ \ \ \ \ \   \ge \frac{(C_\rho-1)L_g +(C_\rho +1)\mu }{2} ||E_t||_F\\
&\ \ \ \ \ \ \ \ \ \  \overset{(a)}{\ge}     \underbrace{\frac{(C_\rho-1)L_g +(C_\rho +1)\mu }{\sqrt{2L_g}} }_{\zeta}\left(g(L_{t}R_t^\top)-g(X_\star)\right)^{\frac{1}{2}} ;\\
&||\nabla g(L_{t+1}R_t^\top)^\top L_{t+1} (L_{t+1}^\top L_{t+1})^{\dagger/2} ||_F \\
&\ \ \ \ \ \ \ \ \ \  \ge \frac{(C_\rho-1)L_g +(C_\rho +1)\mu }{2} ||E_{t+\frac{1}{2}}||_F\\
&\ \ \ \ \ \ \ \ \ \  \overset{(a)}{\ge}  \underbrace{\frac{(C_\rho-1)L_g +(C_\rho +1)\mu }{\sqrt{2L_g}} }_{\zeta}\left(g(L_{t+1}R_t^\top)-g(X_\star)\right)^{\frac{1}{2}},
\end{aligned}
\notag \end{equation}
 where $(a)$ uses the result of equation (\ref{equ:37}).
Therefore, we complete the proof of Lemma \ref{lemma:12}
\end{IEEEproof}

By combining the results of Lemma \ref{lemma:lipschitz-general} and Lemma \ref{lemma:12}, we obtain 
\begin{equation}
\begin{aligned}
& g(L_{t+1}R_t^\top) - g(X_\star) \\
& \le g(L_tR_t^\top) - g(X_\star) - \eta(1-\frac{L_g\eta}{2}) ||\nabla g(L_tR_t^\top)R_t (R_t^\top R_t)^{\dagger/2} ||_F^2\\
&\le g(L_tR_t^\top) - g(X_\star) - \eta(1-\frac{L_g\eta}{2}) \zeta^2 \left(g(L_{t}R_t^\top)-g(X_\star)\right)   \\
&\le  \left(1-\eta(1-\frac{L_g\eta}{2}) \zeta^2\right)  \left(g(L_{t}R_t^\top)-g(X_\star)\right),\\
& g(L_{t+1}R_{t+1}^\top) -g(X_\star) \le  \left(1-\eta(1-\frac{L_g\eta}{2}) \zeta^2\right) \left(g(L_{t+1}R_{t}^\top) -g(X_\star)\right),
\end{aligned}
\notag
\notag \end{equation}
which leads to
$$
g(L_{t+1}R_{t+1}^\top) -g(X_\star) \le \left(1-\eta(1-\frac{L_g\eta}{2}) \zeta^2\right)^2 \left(g(L_{t}R_t^\top)-g(X_\star)\right).
$$
Therefore, we complete the proof of Theorem \ref{theorem:general}.
\end{CJK}
\end{document}